\documentclass[11pt]{article}

\title{Online learning over a finite action set with limited switching}

\usepackage{amsmath,amssymb,graphicx,url}
\usepackage{graphicx}
\usepackage{subfigure}
\usepackage[boxed]{algorithm2e}

\usepackage{amsthm}
\newtheorem{theorem}{Theorem}
\newtheorem{corollary}[theorem]{Corollary}
\newtheorem{lemma}[theorem]{Lemma}
\newtheorem{proposition}[theorem]{Proposition}

\newtheorem{remark}[theorem]{Remark}

\usepackage{times}
\usepackage{enumitem}
\usepackage{framed}
\usepackage{multirow}
\usepackage{bm}
\usepackage{float}
\usepackage{color}

\usepackage[round,comma]{natbib}
\usepackage[margin=1.3in]{geometry} 

\usepackage{hyperref}       
\usepackage{url}            
\usepackage{booktabs}       
\usepackage{nicefrac}       
\usepackage{microtype}      
\usepackage{framed}

\DeclareMathOperator{\Real}{\mathbb{R}}
\DeclareMathOperator*{\argmin}{arg\,min}

\DeclareMathOperator{\E}{\mathbb{E}}

\DeclareMathOperator{\Prob}{\mathbb{P}}
\DeclareMathOperator{\half}{\tfrac{1}{2}}

\DeclareMathOperator{\Ber}{\text{Ber}}
\DeclareMathOperator{\Bin}{\text{Bin}}
\DeclareMathOperator{\Exp}{\text{Exp}}

\DeclareMathOperator{\eps}{\varepsilon}
\DeclareMathOperator{\logdel}{\log\tfrac{1}{\delta}}
\DeclareMathOperator{\logtdel}{\log\tfrac{2}{\delta}}
\DeclareMathOperator{\halfdel}{\tfrac{\delta}{2}}

\DeclareMathOperator{\asleq}{\overset{\text{a.s.}}{\leq}}
\DeclareMathOperator{\asgeq}{\overset{\text{a.s.}}{\geq}}

\DeclareMathOperator{\Regret}{\text{Regret}}
\DeclareMathOperator{\Regretunclipped}{\text{Regret}^{\text{unclipped}}}
\DeclareMathOperator{\elltunclipped}{\ell_t^{\text{unclipped}}}
\DeclareMathOperator{\hatl}{\hat{\ell}}

\DeclareMathOperator{\Switches}{\text{Switches}}
\DeclareMathOperator{\Ltarg}{L_\text{target}}

\DeclareMathOperator{\pfe}{\text{PFE}}
\DeclareMathOperator{\mab}{\text{MAB}}

\DeclareMathOperator{\sd}{\textsc{SD}}

\DeclareMathOperator{\btl}{\textsc{BTL}}
\DeclareMathOperator{\ftl}{\textsc{FTL}}

\DeclareMathOperator{\fpl}{\textsc{FPL}}

\DeclareMathOperator{\mfpl}{\textsc{FPL}^*}
\DeclareMathOperator{\mfpleps}{\textsc{FPL}_{\eps}^*}
\DeclareMathOperator{\bmfpl}{\textsc{BFPL}_{\delta}^*}

\DeclareMathOperator{\pr}{\textsc{PRW}}
\DeclareMathOperator{\bpr}{\textsc{BPRW}}

\DeclareMathOperator{\cpr}{\textsc{CombPRW}}
\DeclareMathOperator{\bcpr}{\textsc{BCombPRW}}

\DeclareMathOperator{\calA}{\mathcal{A}}
\DeclareMathOperator{\calB}{\mathcal{B}}
\DeclareMathOperator{\calF}{\mathcal{F}}

\newcommand{\plusminus}{\raisebox{.2ex}{$\scriptstyle\pm$}}

\newcommand{\freefootnote}[1]{{%
  \let\thempfn\relax
  \footnotetext[0]{\emph{#1}}
}}

\author{
        Jason Altschuler\\
        Massachusetts Institute of Technology\\
        \texttt{jasonalt@mit.edu}
        \and
        Kunal Talwar\\
        Google Brain\\
        \texttt{kunal@google.com}
}

\begin{document}

\date{}
\maketitle

\begin{abstract}
This paper studies the value of switching actions in the Prediction From Experts (PFE) problem and Adversarial Multi-Armed Bandits (MAB) problem. First, we revisit the well-studied and practically motivated setting of PFE with switching costs. Many algorithms are known to achieve the minimax optimal order of $O(\sqrt{T \log n})$ in \textit{expectation} for both regret and number of switches, where $T$ is the number of iterations and $n$ the number of actions. However, no \textit{high probability} guarantees are known. Our main technical contribution is the first algorithms which with high probability achieve this optimal order for both regret and number of switches. This settles an open problem of~\citep{DevLugNeu15}, directly implies the first high probability guarantees for several problems of interest, and is efficiently adaptable to the related problem of online combinatorial optimization with limited switching.
\par Next, to investigate the value of switching actions at a more granular level, we introduce the setting of \textit{switching budgets}, in which the algorithm is limited to $S \leq T$ switches between actions. This entails a limited number of free switches, in contrast to the unlimited number of expensive switches allowed in the switching cost setting. Using the above result and several reductions, we unify previous work and completely characterize the complexity of this switching budget setting up to small polylogarithmic factors: for both the PFE and MAB problems, for all switching budgets $S \leq T$, and for both expectation and high probability guarantees. For PFE, we show that the optimal rate is of order $\tilde{\Theta}(\sqrt{T\log n})$ for $S = \Omega(\sqrt{T\log n})$, and $\min(\tilde{\Theta}(\tfrac{T\log n}{S}), T)$ for $S = O(\sqrt{T \log n})$. Interestingly, the bandit setting does not exhibit such a phase transition; instead we show the minimax rate decays steadily as $\min(\tilde{\Theta}(\tfrac{T\sqrt{n}}{\sqrt{S}}), T)$ for all ranges of $S \leq T$. These results recover and generalize the known minimax rates for the (arbitrary) switching cost setting.
\end{abstract}


\freefootnote{Accepted for presentation at Conference on Learning Theory (COLT) 2018.}

\newpage
\tableofcontents


\section{Introduction}\label{sec:intro}

Two fundamental problems in online learning are the \textit{Prediction From Experts (PFE)} problem ~\citep{CB-expert,Book-CB-Lugosi} and the Adversarial \textit{Multi-Armed Bandit (MAB)} problem~\citep{Auer02,Bubecksurvey}. Over the past few decades, these problems have received substantial attention due to their ability to model a variety of problems in machine learning, sequential decision making, online combinatorial optimization, online linear optimization,
mathematical finance, and many more.
\par PFE and MAB are typically introduced as $T$-iteration repeated games between an algorithm (often called player or forecaster) and an adversary (often called nature). In each iteration $t \in \{1, \dots, T\}$, the algorithm selects an action $i_t$ out of $n$ possible actions, while the adversary simultaneously chooses a loss function over the actions $\ell_t : \{1, \dots, n\} \to [0,1]$. The algorithm then suffers the loss $\ell_t(i_t)$ for its action.
\par For concreteness, consider the following classic example: at the beginning of each day (the iterations) we choose one of $n$ financial experts (the actions), and invest based on her advice about how the stock market will move that day. At the end of the day, we lose (or gain) money based on how good her advice was. Here, the adversary can be thought of as the stock market, since it determines the losses.
\par The goal of the algorithm is to minimize its cumulative loss $\sum_{t=1}^T \ell_t(i_t)$ over the course of the game. In our running example, this corresponds to the total amount of money we lose throughout the investment period. Of course, this cumulative loss can be arbitrarily and hopelessly bad since the choice of losses are at the adversary's disposal. As such, one measures the cumulative loss of the algorithm against a more meaningful baseline: the cumulative loss of the \textit{best action in hindsight}. The algorithm's \textit{regret} is defined as the difference between these two quantities:

\[
\Regret := \sum_{t=1}^T \ell_t(i_t) - \min_{i^* \in [n]} \sum_{t=1}^T \ell_t(i^*)
\]
Indeed in most applications, we typically have reason to believe that at least one action will be decent throughout the game; otherwise there is nothing to learn. Informally, regret measures the algorithm's ability to learn this best action from adversarially noisy observations. Note that if an algorithm achieves sublinear regret ($\Regret = o(T)$) as a function of the horizon $T$, then its average performance converges to that of the best action in hindsight.
\par The PFE and MAB problems differ in the feedback that the algorithm receives. In PFE, the algorithm is given \textit{full-information feedback}: after the $t$th iteration it can observe the entire loss function $\ell_t$. However in MAB, the algorithm is only granted \textit{bandit feedback}: after the $t$th iteration, it can only observe the loss $\ell_t(i_t)$ of the action $i_t$ it played.
\par In our running example, full-information feedback corresponds to observing how good each of the experts' advice was, whereas with bandit-feedback we observe only that for the expert whose advice we actually took. Both of these settings occur in real life: the former when the experts' advice are displayed publicly (e.g. stock predictions on TV or money.cnn.com); and the latter when we have to pay for each expert's advice (e.g. hedgefunds or private wealth management advisories).
\paragraph*{Switching as a resource.} Note that in the setup of PFE and MAB above, the algorithm is not penalized for switching between different actions in consecutive iterations. However, in many practical applications it is beneficial to switch infrequently. For instance, in our earlier example, switching financial advisors between consecutive days could incur many negative consequences~\citep{DekDinKorPer}, such as the cost of cancelling a contract with the last advisor, the cost of signing a contract with the new advisor, the cost of re-investing according to the new advisor's suggestions, or acquiring a bad reputation that makes advisors reluctant to negotiate with you in the future.
\par Infrequent switching is desirable in many other real-world problems. One example is the online shortest paths problem, which has been studied intensely in both the full-information~\citep{TakWar03, KalVem, AweKle08, hedging-structured-concepts} and bandit settings~\citep{AbeHazRak09, combinatorial-bandits}. In this problem, there is a static underlying graph and designated source and sink nodes $s$ and $t$; however, edge weights change adversarially each iteration. The algorithm chooses at each iteration a path from $s$ to $t$, and incurs as loss the weight of that path. A common application of this problem is to select the fastest routes for packets to traverse over the internet. Switching actions thus corresponds to changing routes for packets, which
can lead to notoriously difficult problems in networking such as out-of-order delivery of packets and decoding errors on the receiving end~\citep{SDN, DevLugNeu15}. A related problem is that of learning spanning trees due to its connection to the Internet Spanning Tree Protocol (IEEE 802.1D)~\citep{hedging-structured-concepts, combinatorial-bandits}. Consider also the online learning of permutations~\citep{permelearn, hedging-structured-concepts} which can model online job scheduling in factories; there, a switch could correspond to the laborious and expensive task of modifying the assembly line.
\par Other applications known to benefit from infrequent switching include: the tree-update and list-update problems~\citep{list-update-problem, tree-update-problem, KalVem}; online pruning of decision trees and decision graphs~\citep{prune-decision-tree,prune-decision-graph, KalVem}; learning rankings and online advertisement placement~\citep{Sha12, AudBubLug13}; the Adaptive Huffman Coding problem~\citep{adaptive-huffman-coding, KalVem}; learning adversarial Markov Decision Processes~\citep{mdp-switching, EveKakMan09, Neu10}; the online buffering problem and economical caching~\citep{Dartboard}; and the limited-delay universal lossy source coding problem~\citep{GyoNeu14}.
\par These myriad applications motivate the idea of switching as a resource. This notion has attracted significant research interest in the past few years. The popular way to formalize this idea is the $c$-\textbf{\textit{switching-cost}} setting, in which the algorithm incurs an additional loss\footnote{For simplicity, we assume $c \geq 1$; all upper bounds throughout the paper hold with $c$ replaced by $\max(c, 1)$.} of $c \geq 1$ each time it switches actions in consecutive iterations. In this paper, we introduce the $S$-\textbf{\textit{switching-budget}} setting, in which the algorithm can switch at most $S \in \{1, \dots, T\}$ times in the $T$ iterations. In words, the \textit{switching-cost setting corresponds to expensive but unlimited switches}; whereas the \textit{switching-budget setting corresponds to free but limited switches}.

\paragraph*{Remark about the power of the adversary.} It is important to clarify whether the adversary is allowed to select its loss function $\ell_t$ on the $t$th iteration as a function of the player's previous actions $\{i_s\}_{s=1}^{t-1}$. If yes, then the adversary is said to be \textit{adaptive}; if no, then the adversary is said to be \textit{oblivious} (to the player's actions). Note that without loss of generality, we may assume an oblivious adversary selects all loss functions before the game begins.
\par In the classical unconstrained-switching setting, one can learn well against both adaptive and oblivious adversaries, in both MAB and PFE. However, once we penalize switching (with either switching costs or switching budgets), adaptive adversaries are too powerful and can force any algorithm to incur linear regret (details in Appendix~\ref{app:adaptive}). 
As such, the rest of the paper focuses only on the oblivious adversarial model.

\subsection{Previous work}\label{subsec:previous-work}
The inherent complexity of an online learning problem is typically characterized in terms of the optimal order of growth of regret. This is formalized by the 
\textit{minimax rate}, which is defined as the infimum over (possibly randomized) algorithms, of the supremum over (possibly randomized) adversaries, of expected regret. 
\paragraph*{Previous work on Prediction from Experts (details in Figure~\ref{table:pfe}).} In the classical (unconstrained) setting, the minimax rate $\Theta(\sqrt{T\log n})$ is well understood~\citep{Lit94,FreSch97,CB-expert}. Moreover, this optimal regret rate is also achievable with high probability~\citep{Book-CB-Lugosi}.
\par The minimax rate is also well-understood in the $c$-switching cost setting. Recall that here the objective is ``switching-cost-regret'', which is defined as $\Regret +\,c\cdot(\text{\# switches})$. The minimax rate for expected switching-cost-regret is $\Theta(\sqrt{cT \log n})$ for PFE~\citep{KalVem,Dartboard,DevLugNeu15}. In particular, these results give algorithms which achieve the optimal minimax order in \textit{expectation} for both regret and number of switches.
\par Surprisingly, however, \textit{no high-probability guarantees are known for switching-cost PFE}; this is raised as an open question by~\citep{DevLugNeu15}. Indeed, all algorithms that work well in expectation have upper tails that are either too large or unknown how to analyze.
In the first category is the Multiplicative Follow the Perturbed Leader ($\mfpl$) algorithm of~\citep{KalVem}: the large upper tails of this algorithm are folklore but we are not aware of an explicit reference; for completeness, we give a proof in Appendix~\ref{app:lb-kv-sd}.
In the second category are the Prediction by Random-Walk Perturbation ($\pr$) algorithm of~\citep{DevLugNeu15} and the Shrinking Dartboard ($\sd$) algorithm of ~\citep{Dartboard}. Analyzing the upper tails of these algorithms seem difficult and was left as an open question for $\pr$ in~\cite{DevLugNeu15}. For $\sd$, we show in Appendix~\ref{app:lb-kv-sd} that the tails are at best sub-exponential, and are thus strictly worse than sub-Gaussian (which our algorithms achieve, and which we show is optimal via a matching lower bound).

\par For the $S$-switching-budget setting, even less is known. The only relevant lower bound seems to be the trivial one that expected regret is $\Omega(\sqrt{T \log n})$ for all $S$, which is the lower bound for the unconstrained setting~\citep{CB-expert} and applies since constraining switching obviously can only make the problem harder. The only relevant upper bound is due to the Lazy Label Efficient Forecaster (combined with a simple reduction). However, although this elegant algorithm yields tight bounds for the (very different) setting of Label Efficient Prediction it was designed for~\citep{lazy-label-efficient-forecaster}, it achieves only $O(\tfrac{T}{\sqrt{S}})$ regret for our setting, which is very far from the $\Theta(\tfrac{T}{\min(S,\sqrt{T})})$ minimax rate we prove in this paper.
\par Note also that the existing minimax-optimal switching-cost algorithms $\mfpl$, $\sd$, and $\pr$ do not apply to the switching-budget setting (even in expectation), since the number of times they switch is only bounded \textit{in expectation}, whereas the switching budget setting requires a hard cap. Of course this could be fixed if the number of switches an algorithm makes has an exponentially small tail (see Section~\ref{sec:budgets-pfe} for details), but this fails for existing algorithms for the reasons stated in the above discussion about high-probability bounds for switching-cost PFE (see also Appendix~\ref{app:lb-kv-sd}).
\paragraph*{Previous work on Multi-Armed Bandits (details in Figure~\ref{table:mab}).} In the unconstrained setting, the minimax rate $\Theta(\sqrt{Tn})$ is well understood~\citep{Auer02,AudBub10} and is achieveable with high probability~\citep{AudBub10, Bubecksurvey}. 

\par For the $c$-switching cost setting, the minimax rate is known (up to a logarithmic factor in $T$) to be $\tilde{\Theta}(c^{1/3}T^{2/3}n^{1/3})$ for MAB~\citep{AroDekTew12,DekDinKorPer}. We note that high-probability guarantees are not explicitly written in the literature, but can be easily obtained by combining the high-probability guarantee of an algorithm designed for unconstrained MAB, with a standard mini-batching reduction (e.g. the one by~\citep{AroDekTew12}). Using the best known bound for unconstrained MAB (achieved by the Implicitly Normalized Forecaster of~\citep{AudBub10}) yields the bound in Table~\ref{table:mab}. For completeness we give details in Section~\ref{sec:budgets-mab}.
\par For the $S$-switching budget setting, a similar simple mini-batching reduction gives algorithms achieving the minimax rate in expectation and with high probability (details in  Section~\ref{sec:budgets-mab}). The lower bound for this setting is significantly more involved. The relevant result is Theorem 4 of~\citep{DekDinKorPer} who prove a lower bound of $\tilde{\Omega}(\tfrac{T}{\sqrt{S}})$ via a reduction to the switching-cost setting. However, their reduction does not get the correct dependence on the number of actions $n$ and also loses track of polylogarithmic factors.

\begin{table}[H]
\centering
\vspace{-0.5mm}
\caption{Upper and lower bounds on the complexity of PFE in the different switching settings. Our new bounds are bolded.}
\vspace*{2mm}
\label{table:pfe}
\begin{tabular}{|c|c|c|c|}
\hline
          & \textbf{LB on $\E[\Regret]$} & \textbf{UB on $\E[\Regret]$} & \textbf{High probability UB} \\ \hline
\textbf{Unconstrained switching} & $\sqrt{T \log n}$  & $\sqrt{T \log n}$ & $\sqrt{T \log \tfrac{n}{\delta}}$ \\ \hline
\textbf{$c$ switching cost}      & $\sqrt{cT \log n}$ & $\sqrt{c T \log n}$ &     $\pmb{\sqrt{cT \log n \logdel}}$ \\ \hline
\textbf{$S=\Omega(\sqrt{T \log n})$ switching budget}& $\sqrt{T \log n}$& $\pmb{\sqrt{T \log n} \log T}$  & $\pmb{\sqrt{T \log n} \logdel}$   \\ \hline
\textbf{$S=O(\sqrt{T \log n})$ switching budget}& $\pmb{\frac{T \log n}{S}}$ & $\pmb{\frac{T \log n}{S}\log T}$ & $\pmb{\frac{T\log n}{S}\logdel}$ \\ \hline
\end{tabular}
\end{table}

\begin{table}[H]
\centering
\vspace{-2mm}
\caption{Upper and lower bounds on the complexity of MAB in the different switching settings. Our new bounds are bolded.}
\vspace*{2mm}
\label{table:mab}
\begin{tabular}{|c|c|c|c|}
\hline
 & \textbf{LB on $\E[\Regret]$} & \textbf{UB on $\E[\Regret]$} & \textbf{High probability UB} \\ \hline
\textbf{Unconstrained switching} & $\sqrt{Tn}$& $\sqrt{Tn}$ & $\sqrt{Tn} \frac{\log \frac{n}{\delta}}{\sqrt{\log n}}$ \\ \hline
\textbf{$c$ switching cost}      &
$\frac{c^{1/3}T^{2/3}n^{1/3}}{\log T}$ & $c^{1/3}T^{2/3}n^{1/3}$ & $c^{1/3}T^{2/3}n^{1/3} \frac{\log^{2/3} \frac{n}{\delta}}{\log^{1/3} n}$ \\ \hline
\textbf{$S$ switching budget}& $\pmb{\frac{T\sqrt{n}}{\sqrt{S} \log^{3/2} T}}$ & ${\frac{T\sqrt{n}}{\sqrt{S}}}$  & ${\frac{T\sqrt{n}}{\sqrt{S}} \frac{\log\frac{n}{\delta}}{\sqrt{\log n}}}$  \\ \hline
\end{tabular}
\end{table}

\begin{figure}
\centering     
\subfigure[Switching-budget PFE.]{\label{fig:pfe}\includegraphics[width=70mm]{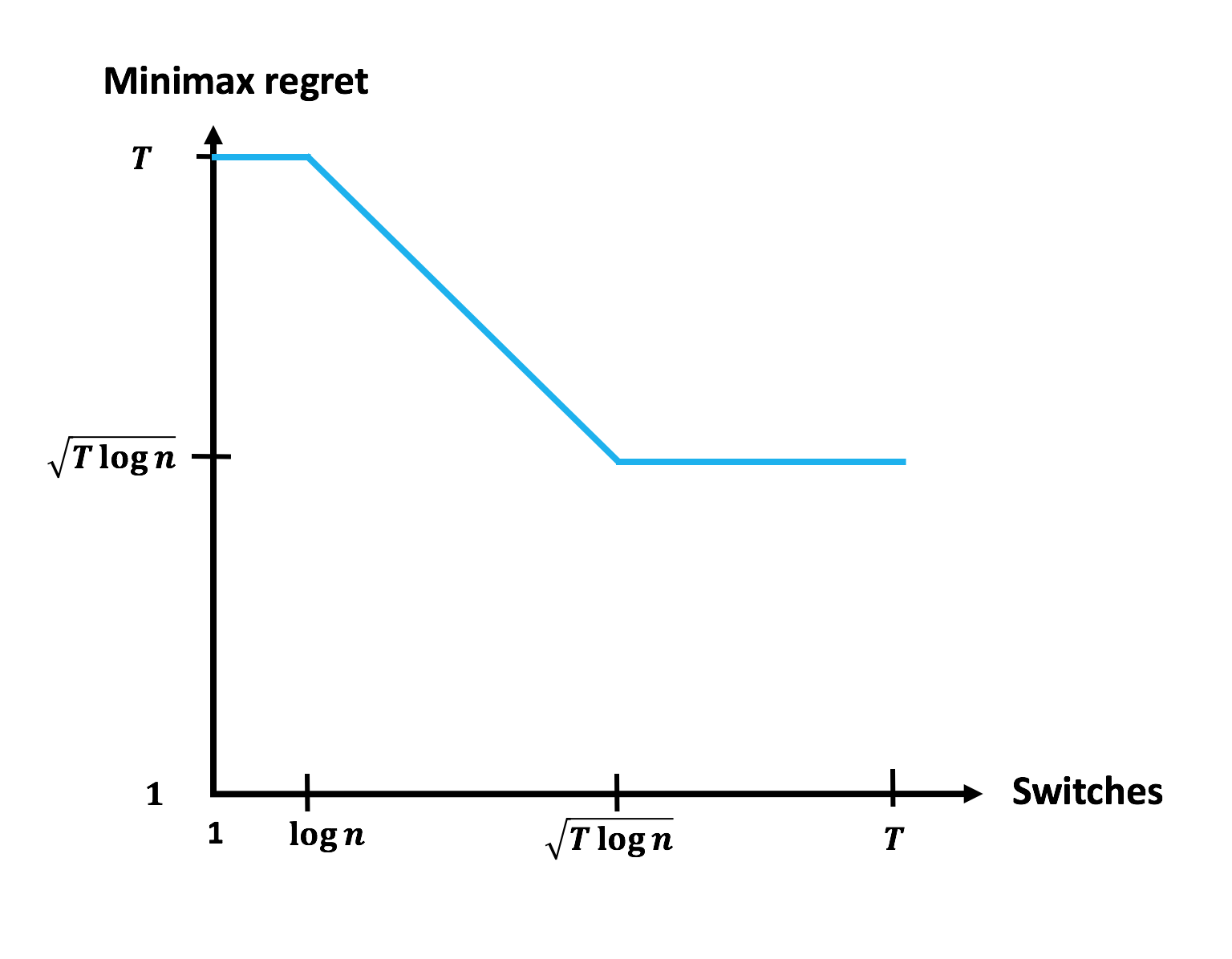}}
\subfigure[Switching-budget MAB.]{\label{fig:mab}\includegraphics[width=70mm]{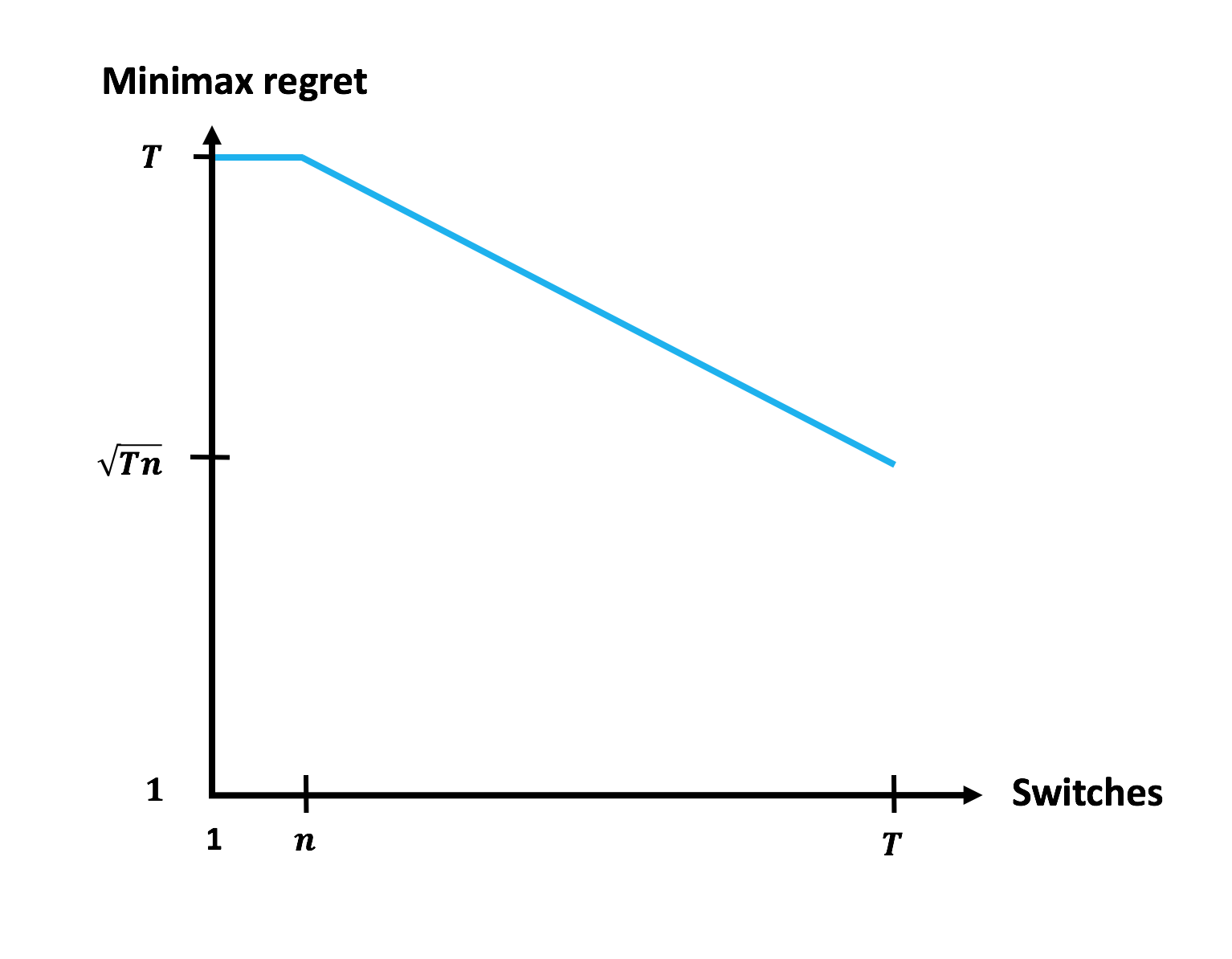}}
\caption{Complexity landscape of online learning over a finite action set with limited switching. Axes are plotted in log-log scale. Polylogarithmic factors in $T$ are hidden for simplicity.}
\label{fig:plots}
\end{figure}


\subsection{Our contributions}\label{subsec:intro-contributions}

In Section~\ref{sec:hp}, we present the first algorithms for switching-cost PFE that achieve the minimax optimal rate $O(\sqrt{cT\log n})$ with high probability, settling an open problem of~\citep{DevLugNeu15}. In fact, our results are more general: we give a framework to formulaically convert algorithms that work in expectation and fall under the Follow-the-Perturbed-Leader algorithmic umbrella, into algorithms that work with high probability. We then present our algorithms as examples of this framework. We also show how this framework extends to online combinatorial optimization\footnote{I.e. online linear optimization over a combinatorial polytope, where offline optimization can be done efficiently.} with limited switching, and give the first high-probability algorithm for this problem.\footnote{Note that online combinatorial optimization can be recast na\"ively as PFE where each vertex is modeled by an expert, but then the runtime of each iteration is linear in the number of vertices, which is typically exponential in the dimension (see e.g.~\citep{hedging-structured-concepts}). Various approaches have been developed to overcome this; we show that our framework also applies to this setting without modification.}

\par Next, to investigate the value of switching actions at a more granular level, we study the new setting of \textit{switching budgets} for the PFE and MAB problems, respectively in Sections~\ref{sec:budgets-pfe} and~\ref{sec:budgets-mab}. The above result and standard reductions allow us to completely characterize the complexity of this switching budget setting up to small polylogarithmic factors: for both the PFE and MAB problems, for all switching budgets $S \leq T$, and for both expectation and high probability guarantees. For PFE, we show the optimal rate is of order $\tilde{\Theta}(\sqrt{T\log n})$ for $S = \Omega(\sqrt{T\log n})$, and $\min(\tilde{\Theta}(\tfrac{T\log n}{S}), T)$ for $S = O(\sqrt{T \log n})$. Interestingly, the bandit setting does not exhibit such a phase transition; instead we show the minimax rate decays steadily as $\min(\tilde{\Theta}(\frac{T\sqrt{n}}{\sqrt{S}}), T)$ for all ranges of $S \leq T$.


\subsection{Notation}\label{subsec:notation}
We denote the set of integers $\{1, \dots, k\}$ by $[k]$. For shorthand, we abbreviate ``almost surely'' by a.s., ``independently and identically distributed'' by i.i.d., ``with respect to'' by w.r.t., ``without loss of generality'' by WLOG, ``random variable'' by r.v., ``(with) high probability'' by (w.)h.p., and ``left (resp. right) hand side'' by LHS (resp. RHS). We denote Bernoulli, binomial, and exponential distributions by $\Ber(\cdot)$, $\Bin(\cdot, \cdot)$, and $\Exp(\cdot)$, respectively. We write $X \sim D$ to denote that the r.v. $X$ has distribution $D$, and we write $\sigma(X_1, \dots, X_k)$ to denote the sigma-algebra generated by the random variables $X_1, \dots, X_k$. We denote the total variation distance between two probability measures $P$ and $Q$ w.r.t. a $\sigma$-algebra $\calF$ by $\|P - Q\|_{\text{TV}, \calF} := \sup_{A \in \calF}|P(A) - Q(A)|$.
\par Throughout we reserve the variable $T$ for the number of iterations in the game, $n$ for the number of actions (i.e. ``experts'' in PFE or ``arms'' in MAB), and $S$ for the switching budget. Losses of the actions are denoted by $\{\ell_t(i)\}_{t \in [T], i \in [n]}$. Often it is notationally convenient to add in a fake zero-th iteration; when we do this the losses are all zero $\{\ell_0(i) := 0\}_{i \in [n]}$. Next, we write $\Regret_t(\calA)$ and $\Switches_t(\calA)$ to denote the regret and number of switches, respectively, that an algorithm $\calA$ makes in the first $t$ iterations. When there is no chance of confusion, we often write just $\Regret_t$ and $\Switches_t$ for shorthand.
\par \textbf{Remark on integrality.} Since we are interested only in asymptotics and to avoid carrying ceilings and floors throughout, we ignore issues of integrality for notational simplicity.


\section{Switching-cost PFE: the first high probability algorithms}\label{sec:hp}
The main result of this section is the following. A partial converse is presented in Appendix~\ref{app:subsec:lb-regret}.

\begin{theorem}\label{thm:hp-sc}
For each $\delta \in (0, \half)$, there exists an algorithm for PFE satisfying the following against any oblivious adversary
\[
\Prob\left(
\Regret_T,\; \Switches_T \leq O\left(\sqrt{T\log n\logdel} \right)
\right)
\geq 1 - \delta
\] 
\end{theorem}
Note that Theorem~\ref{thm:hp-sc} does not give an \textit{uniform} algorithm; that is, for each failure probability $\delta$, there is a (possibly different) corresponding algorithm. Obtaining a uniform h.p. algorithm remains an important open question.
\par Combining Theorem~\ref{thm:hp-sc} with a standard mini-batching argument (see e.g.~\citep{AroDekTew12}) immediately yields the first h.p. algorithm for the (arbitrary) switching-cost PFE problem.
\begin{corollary}
Let $c \geq 1$. For each $\delta \in (0, \half)$, there exists an algorithm for PFE satisfying the following against any oblivious adversary
\[
\Prob\left(
\Regret_T +\, c \cdot \Switches_T \leq O\left(\sqrt{cT\log n\logdel} \right)
\right)
\geq 1 - \delta
\]
\end{corollary}
\begin{proof}
Mini-batch the $T$ iterations into $\tfrac{T}{c}$ contiguous epochs of length $c$, and then apply the algorithm from Theorem~\ref{thm:hp-sc}.
\end{proof}


\par The remainder of this section is organized as follows. Subsection~\ref{subsec:hp-framework} describes a general framework for producing a h.p. algorithm given an algorithm satisfying certain properties. In Subsection~\ref{subsec:hp-switching}, we show that the switching cost bound holds w.h.p. under our framework. Our analysis of the regret bounds is not blackbox; the rest of this section analyzes regret bounds for three algorithms in the $\fpl$ algorithmic umbrella.

\subsection{Framework for converting $\fpl$-based algorithms with expectation guarantees into ones with high-probability guarantees}\label{subsec:hp-framework}
We will shortly give several algorithms achieving the desired h.p. guarantees in Theorem~\ref{thm:hp-sc}. All of our algorithms use the same general idea, and so we first abstract slightly and describe a meta-framework for constructing such h.p. algorithms. Informally, the idea for ``boosting'' our success probability is to repeatedly run an algorithm $\calA$ which is optimal in expectation. This allows a user to easily and formulaically construct h.p. algorithms specially tailored to desired applications, since proving expectation bounds is significantly simpler than proving h.p. bounds.

\par Algorithm~\ref{alg:framework} informally describes this simple meta framework. The idea is to split the $T$ iterations into $N = \logtdel$ variable-length epochs. In each epoch, we restart and run the subroutine $\calA$ until it uses $S' \asymp \sqrt{T\log n \logdel}/N$ switches.
\par We emphasize two critical aspects of this meta-framework. First, the subroutine $\calA$ is \textit{restarted with fresh randomness in each new epoch}. This is clearly essential for concentration. Second, the epochs are of \textit{variable length} (and are in fact random). This can be shown to be essential: if $\calA$ has large upper tails on switching (see Appendix~\ref{app:lb-kv-sd}), then simply running $\calA$ roughly $\approx \logdel$ times in 
$\approx T/\logdel$
consecutive epochs of fixed uniform size does not imply the desired concentration in the total number of switches. Indeed, with a probability too large for our desired h.p. bounds, $\calA$ uses far too many switches in one of the epochs, thereby ruining the total switching budget. Further details for this can be found in Appendix~\ref{app:need_properties}.
\begin{figure}
\caption{Framework for obtaining algorithms achieving optimal switching and regret w.h.p.}
\vspace{2mm}
\begin{algorithm}[H]
\While{in iteration $\leq T$}{
Run $\calA$ with fresh randomness. Stop when use $S' = O\left(\sqrt{\frac{T \log n}{\logdel}}\right)$ switches.
}
\label{alg:framework}
\end{algorithm}
\end{figure}
\par The analysis then consists of the following two parts (and then taking a union bound):
\begin{enumerate}
\item Show that w.h.p., we never run out of switches. This amounts to showing that the number of epochs is greater than $N$ with probability at most $e^{-N} = \halfdel$.
\item Show that the cumulative regret over the epochs concentrates around $N$ times the expected regret in a single epoch. We take a point at which the CDF has remaining upper tail $\approx \halfdel$.
\end{enumerate}

The analysis of the second step requires specific properties of $\calA$. We will require the following for all $\tau \leq T$:
\begin{enumerate}
\item[(i)] $\E[\Switches_\tau(\calA)] = O\left(\sqrt{\tau \log n}\right)$
\item[(ii)] $\Regret_\tau(\calA) \asleq O\left(\Switches_{\tau}(A)\right) + Z_{\tau}(A)$, where $Z_{\tau}(A)$ is independent of the adversary
\item[(iii)] (Informal) $Z_{\tau}(A)$ has ``exponentially small upper tails''
\end{enumerate}

Let us comment on these properties. Informally, using (i), we can bound the number of switches w.h.p., and then (ii) and (iii) together imply a h.p. bound on the regret. We will show the former can be proven in a blackbox manner just given (i); however, the latter requires analyzing $Z_{\tau}(A)$ which is algorithm-specific.
\par Note that property (ii) enforces a certain dependence between the regret and number of switches. Upon first sight it may seem restrictive or slightly unusual; however, this is actually a natural property of $\fpl$-based algorithms. In particular, we show in the following subsections how this applies to~\citep{KalVem}'s Multiplicative Follow the Perturbed Leader algorithm ($\mfpl$) and~\citep{DevLugNeu15}'s Prediction by Random Walk Perturbation algorithm ($\pr$).

\subsection{Black-box high probability bounds on switching}\label{subsec:hp-switching}
In this subsection, we formalize the above discussion about how property (i) (upper bound on expected number of switches) is sufficient to prove in an entirely black-box manner that Framework~\ref{alg:framework} produces an algorithm with h.p. bounds on switching. We also isolate in the statement of this lemma an upper tail bound on the number of epochs in the batched algorithm, since this will later be useful for proving h.p. regret bounds. We state the result in some generality, assuming that the algorithm $\calA$ has expected switches at most $c_{\calA} \sqrt{T}$; for example $c_{\calA} = c \sqrt{\log n}$ for $\mfpl$.

\begin{lemma}\label{lem:switching-bb}
Let $\calA$ be an algorithm for PFE satisfying $\E[\Switches_\tau(\calA)] \leq c_{\calA} \sqrt{\tau}$ for all $\tau \leq T$ and some parameter $c_{\calA} > 0$  independent of $\tau$. Denote by $\calB$ the algorithm produced by Framework~\ref{alg:framework} with $S' := 23c_{\calA}\sqrt{\frac{T}{\logtdel}}$. Then for any oblivious adversary, the number of epochs $E$ in $\calB$ satisfies
\[
\Prob\left(E > \logtdel\right) \leq \halfdel
\]
and so in particular (since each epoch uses at most $S'$ switches),
\[
\Prob\left(
\Switches_T(\calB) > 23c_{\calA}\sqrt{T \logtdel}
\right)
\leq \halfdel
\]
\end{lemma}

\paragraph*{Proof sketch.} 
Note that $E > N := \logtdel$ if and only if $\calB$ has not yet reached iteration $T$ by the end of the first $N$ epochs. Thus it suffices to show 
\begin{align}
\Prob\left(
\sum_{e=1}^{N} L_e < T
\right) \leq e^{-N}
\label{eq:epoch-lengths}
\end{align}
where $L_e$ denotes the length of epoch $e$, and we use the convention $L_e = 0$ if $e > E$. We will prove~\eqref{eq:epoch-lengths} by first using the expected switching bound on $\calA$ and Markov's inequality to show that \textit{each epoch $e$ has large length $L_e$ with reasonably large probability}; and then using concentration-of-measure to conclude \textit{$\sum_{e=1}^N L_e$ is small only with exponentially small probability in $N$.}
\paragraph*{Technical issues of dependency and overcoming them with martingales.} However, there is an annoying technical nuance that must be accounted for in both these steps: $\{L_e\}_{e \in [N]}$ are \textit{dependent}. One reason for this is that the final epoch may be truncated if it reaches iteration $T$ before it uses $S'$ switches; and then all subsequent epochs $e \in \{E+1, \dots, N\}$ will necessarily have length $0$.
\par To fix this issue, consider possibly \textit{extending the game past iteration $T$ (only for analysis purposes) until $N$ epochs are completed}. Set to zero all losses in the ``overtime'' extension of the game past iteration $T$. Define now $\tilde{L}_e$ to be the length of epoch $e$ in the \textit{extended} game. Importantly, observe that the event in~\eqref{eq:epoch-lengths} can be expressed in terms of $\{\tilde{L}_e\}_{e \in [N]}$ as
\begin{align}
\left\{\sum_{e=1}^{N} L_e < T \right\} = \left\{\sum_{e=1}^N \tilde{L}_e < T \right\}
\label{eq:L-Ltilde}
\end{align}
This is because $L_e \overset{\text{a.s.}}{=} \tilde{L_e}$ conditional on the event that none of epoch $e$ goes into overtime.
\par Note that $\{\tilde{L}_e\}_{e \in [N]}$ are still not independent, however they are ``independent enough'' for us to accomplish the proof sketch using martingale concentration. The following lemma formally accomplishes the first step in the above proof sketch: conditional on any history, each $\tilde{L}_e$ has large length $\Ltarg := \frac{8T}{N}$ with reasonably large probability.\footnote{Assume for notational simplicity that $\Ltarg := \frac{8T}{N} \leq T$; otherwise set $\Ltarg$ to $T$ and an identical argument proceeds.}
\begin{lemma}\label{lem:switching-helper}
For each epoch $e \in [N]$ and for each event $R$ in the sigma-algebra $\sigma(\tilde{L}_1, \dots, \tilde{L}_{e-1})$,
\[
\Prob\left(
\tilde{L}_e \geq \Ltarg \; \Big| \; R
\right)
\geq
\frac{7}{8}
\]
\end{lemma}
\begin{proof}
Conditioning on the possible realizations $t_e$ of $\sum_{i < e} \tilde{L}_i$ yields
\begin{align*}
\Prob\left(
\tilde{L}_e \geq \Ltarg \; \Big| \; R
\right)
&=
\sum_{t_e}
\Prob\left(
\tilde{L}_e \geq \Ltarg \; \Big| \; \sum_{i < e} \tilde{L}_i = t_e, R
\right)
\Prob\left(	
\sum_{i < e} \tilde{L}_i = t_e \; \Big| \; R
\right)
\end{align*}
Now the obliviousness of the adversary ensures that $\tilde{L}_e$ is conditionally independent of $R$ given $\sum_{i < e} \tilde{L}_i$. Thus it suffices to now show
$
\Prob(
\tilde{L}_e \geq \Ltarg \, | \, \sum_{i < e} \tilde{L}_i = t_e
)
\geq \tfrac{7}{8}
$
for each realization $t_e$.
By Markov's inequality,
\begin{align*}
\Prob\left(
\tilde{L}_e \geq \Ltarg \; \Big| \; t_e
\right)
&= \Prob\left(
\calA \text{ makes } \leq S' \text{ switches in the }\Ltarg\text{ iterations starting after } t_e
\right)
\\ &= 1 - \Prob\left(
\calA \text{ makes } > S' \text{ switches in }\Ltarg\text{ iterations starting after } t_e
\right)
\\ &\geq 1 - \E\left[\Switches_{\Ltarg}(\calA)\right] / S'
\\ &\geq 1 - \left(c_{\calA}\sqrt{\tfrac{8T}{N}}\right)/\left(23c_{\calA}\sqrt{\tfrac{T}{N}}\right)
\geq \frac{7}{8}
\end{align*}
\end{proof}

We are now ready to prove Lemma~\ref{lem:switching-bb}.
\\ \begin{proof}[Proof of Lemma~\ref{lem:switching-bb}]
Define the indicator random variables $X_e := \mathbf{1}(\tilde{L}_e \geq \Ltarg)$. By ~\eqref{eq:L-Ltilde} and then the observation that $\tilde{L}_e$ stochastically dominates $\Ltarg X_e$,
\begin{align}
\Prob\left(\sum_{e=1}^N L_e < T\right) \nonumber
= \Prob\left(\sum_{e=1}^N \tilde{L}_e < T\right)
\leq \Prob\left( 
\sum_{e=1}^N X_e < \frac{T}{\Ltarg} = \frac{N}{8}
\right)\nonumber
\end{align}
We now claim $\{M_k := \sum_{e=1}^k X_e - \frac{7}{8}k\}_{k \in [N]}$ forms a submartingale w.r.t. the filtration $\{\mathcal{F}_k := \sigma(\tilde{L}_1, \dots, \tilde{L}_k)\}_{k \in [n]}$. To see this, observe that $M_k$ is clearly measurable w.r.t. $\mathcal{F}_k$, and also Lemma~\ref{lem:switching-helper} gives $\E[M_{k+1} | \mathcal{F}_{k}] = M_{k} + \left(\E[X_{k+1} | \mathcal{F}_{k}] - \frac{7}{8}\right) \geq M_{k}$. Therefore since also $\{M_k\}$ clearly has $1$-bounded differences (they are cumulative sums of indicator random variables), Azuma-Hoeffding's submartingale inequality upper bounds the above display by
\[
\Prob\left( 
\sum_{e=1}^N X_e < \frac{N}{8}
\right)
\leq \exp\left(
-2 \left(\frac{7}{8} - \frac{1}{8}\right)^2N
\right)
\leq \exp\left(
-N
\right)
= \frac{\delta}{2}\nonumber
\]
\end{proof}

\subsection{Controlling regret for $\fpl$-based algorithms}\label{subsec:hp-ii}
In this subsection, we formalize the discussion from Subsection~\ref{subsec:hp-framework} about how algorithms based on~\citep{KalVem}'s $\fpl$ algorithmic framework automatically satisfy property (ii). As such, we first review the basics of $\fpl$; we highlight only the background relevant for the results in this paper, and refer the reader to~\citep{KalVem,Book-CB-Lugosi,DevLugNeu15} for further details.
\par The Follow The Leader algorithm ($\ftl$) greedily plays at iteration $t \in [T]$ the action $i_t := \argmin_{i \in [n]} \sum_{s=0}^{t-1} \ell_s(i)$ that has been best so far\footnote{It will be notationally convenient to define $\{\ell_0(i) := 0\}_{i \in [n]}$ so that the game has losses for iterations $t \in \{0, \dots, T\}$.}. However, $\ftl$ is well-known to have $\Omega(T)$ regret in the worst case. It turns out one can fix this by ``perturbing'' the losses in a clever way, and then in each iteration following the leader \textit{with respect to the perturbed losses}. This is the $\fpl$ framework. Note the algorithm is then completely determined by its choice of perturbations.
\par Let us formalize this. Before the game starts, the algorithm chooses (random) perturbations $\{P_{t}(i)\}_{t \in [T+1], i \in [n]}$. It then plays $\ftl$ on the perturbed losses $\{\hatl_t(i) := \ell_t(i) + P_{t+1}(i)\}_{t \in \{0, \dots, T\}, i \in [n]}$. That is, at iteration $t \in [T]$ it plays action
\[
i_t := \argmin_{i \in [n]} \sum_{s=0}^{t-1}\hatl_{s}(i)
\]
The analysis of FPL-style algorithms is somewhat formulaic. One bounds the regret in terms of two terms: the number of times the perturbed leader switches (which is the number of switches that $\fpl$ makes!), and the magnitude of the perturbations. In words, the former controls how predictable the algorithm's actions are; and the latter controls how much the algorithm ``deceives'' itself by playing based on inauthentic losses. This illustrates an important tradeoff, which can be made formal since larger perturbations make the perturbed leader switch fewer times.
\par This bound is written formally as follows~\citep{KalVem,DevLugNeu15}. We generalize the statement slightly by introducing the quantity $M := \sup_{\text{actions } i,i', \text{\; loss }\ell} |\ell(i) - \ell(i')|$. The proof is standard via the so-called ``Be The Leader'' lemma from~\citep{KalVem}, and is given in Appendix~\ref{app:betheleader} for completeness.
\begin{lemma}[Standard lemma in analysis of $\fpl$ algorithms]\label{lem:regret-fpl-framework} Consider $\fpl$ with perturbations $\{P_{t}(i)\}$. The following holds pointwise (w.r.t. the randomness of both the algorithm and adversary)
\[
\Regret_T(\fpl) 
\asleq M \cdot  \Switches_T(\fpl) + \left[
\max_{i \in [n]} \sum_{t=1}^{T+1} P_t(i)
\right]
- \sum_{t=1}^{T+1} P_t(i_t)
\]
\end{lemma}

In the notation of property (ii), the two summands in Lemma~\ref{lem:regret-fpl-framework} constitute $Z_T$. This bounds the per-epoch regret of the batched algorithm from Framework~\ref{alg:framework}.  By summing up the above inequality over all epochs, one can bound on the regret over the whole game.

\begin{corollary}\label{cor:regret-fpl-framework}
Let $\calB$ be the algorithm produced from Framework~\ref{alg:framework} applied to $\fpl$ with perturbations $\{P_{t}(i)\}$. Then
\[
\Regret_T(\calB) \asleq
M \cdot \Switches_T(\calB)
+
\left[\sum_{e = 1}^E \max_{i \in [n]} \sum_{t \in e} P_t(i) \right]
- \sum_{t=1}^{T+1} P_t(i_t)
\] 
\end{corollary}
\begin{proof}
Sum Lemma~\ref{lem:regret-fpl-framework} over all epochs $e \in [E]$, and observe 
\[\sum_{e=1}^E \Regret_e(\calB)
= \sum_{e = 1}^E \left(\sum_{t \in e} \ell_t(i_t) - \argmin_{i^* \in [n]} \sum_{t \in e} \ell_t(i^*) \right)
\asgeq \sum_{t=1}^T \ell_t(i_t) - \argmin_{i^* \in [n]} 
\sum_{e = 1}^E \sum_{t \in e}
 \ell_t(i^*)
= \Regret_T(\calB)\]
\end{proof}

\subsection{High probability version of~\citep{KalVem}'s Multiplicative Follow the Perturbed Leader algorithm}\label{subsec:hp-mfpl}
\citep{KalVem}'s Multiplicative Follow the Perturbed Leader algorithm ($\mfpleps$) is a version of $\fpl$ that sets $P_t(i) = 0$ for all $t > 1$, and draws the intial perturbations $P_1(i) := \tfrac{R(i)}{\eps}$ where $R(i) \sim \exp(1)$ are i.i.d. standard exponential random variables and $\eps$ is a parameter.\footnote{Actually this is a slight variation of $\mfpleps$ in the original paper of~\citep{KalVem}. Here the perturbations are drawn from ($\eps$-scaled) exponential distributions rather than ($\eps$-scaled) Laplace distributions. Both are adaptable to h.p. algorithms in identical ways; the presented version just has a slightly simpler analysis.} $\mfpleps$ admits the following guarantees on its expected number of switches.

\begin{lemma}~\citep{KalVem}\label{lem:fpl-switching}
For any $\tau \in \mathbb{N}$ and any oblivious adversary, $\;\;\E[\Switches_\tau(\mfpleps)] \leq \frac{1}{1 - \eps}\left(
\eps \tau + 2 \log n
\right)$.
\end{lemma}
Apply Framework~\ref{alg:framework} to $\mfpleps$ with parameter choices $\eps = \tfrac{1}{2}\sqrt{\tfrac{\log n \logtdel}{T}}$ and $S' = 135\sqrt{\tfrac{T \log n}{\logtdel}}$. We call this new algorithm Batched Multiplicative Follow the Perturbed Leader ($\bmfpl$). 

\begin{theorem}\label{thm:hp-bfpl}
For any $\delta \in (0,\half)$ and any oblivious adversary,
\[
\Prob\left(
\Regret_T(\bmfpl),\; \Switches_T(\bmfpl) \leq O\left(\sqrt{T\log n\logdel} \right)
\right)
\geq 1 - \delta
\]
\end{theorem}

To prove the h.p. bound on regret for Theorem~\ref{thm:hp-bfpl}, we will need a $\mfpl$-specific version of property (iii) from Subsection~\ref{subsec:hp-framework}. Specifically, as we will see shortly, it will suffice to control the upper tail of the sum of the maximum of $n$ i.i.d. standard exponential variables. Each of the maximums is sub-exponential, so the sum of them is also sub-exponential. Formally, the following concentration inequality will be sufficient. A proof via standard Chernoff bounds is given in Appendix~\ref{app:chernoff}. 

\begin{lemma}\label{lem:PEV} Let $N, n \geq 2$. If $\{R_e(i)\}_{e \in [N], i \in [n]}$ are i.i.d. standard exponentials, then
\[
\Prob\left(\sum_{e=1}^N \max_{i \in [n]} R_e(i)
> 6N \log n\right) \leq e^{-N}.
\]
\end{lemma}

\begin{proof}[Proof of Theorem~\ref{thm:hp-bfpl}]
Assume WLOG that $\sqrt{T \log n \logtdel} < T$, otherwise the desired bound is trivially satisfied. This implies in particular that $\eps < \half$ and 
$\sqrt{\log n} < \sqrt{T/\logtdel}$, thus $\E[\Switches_{\tau}(\fpl_{\eps})] \leq 5\sqrt{T \log n \logtdel}$ for all $\tau \leq T$ by Lemma~\ref{lem:fpl-switching}. We conclude by Lemma~\ref{lem:switching-bb} that the event $\{E \leq \logtdel\}$ occurs with probability at least $1 - \halfdel$; and whenever this occurs, $\bmfpl$ makes at most $135\sqrt{T \log n \logtdel}$ switches.
\par We next prove the h.p. regret bound. By Corollary~\ref{cor:regret-fpl-framework} and the choice of perturbations in $\mfpleps$,
\[
\Regret_{T}(\bmfpl)
\asleq \Switches_T(\bmfpl) + 
\sum_{e \in [E]} \max_{i \in [n]} P_e(i)
\]
where $P_e(i) := \frac{R_e(i)}{\eps}$ denotes expert $i$'s perturbation in epoch $e$. The proof is thus complete by taking a union bound over the occurence of $\{E \leq \logtdel\}$ and the event in Lemma~\ref{lem:PEV}.
\end{proof}

\subsection{High probability version of~\citep{DevLugNeu15}'s Prediction By Random Walk Perturbation algorithm}\label{subsec:hp-pr}
The Prediction by Random Walk Perturbation~\citep{DevLugNeu15} algorithm ($\pr$) is a version of $\fpl$ with all perturbations $P_t(i)$ drawn i.i.d. uniformly at random from $\{\plusminus \tfrac{1}{2} \}$.~\citep{DevLugNeu15} show that $\pr$ achieves the optimal order for regret and switching in expectation, and raise h.p. bounds as an open question. We show presently how to achieve h.p. bounds using Framework~\ref{alg:framework}.
\par First, let us recall property (i) for $\pr$, i.e. a bound on its expected number of switches.
\begin{lemma}~\citep{DevLugNeu15}~\label{lem:pr-switching}
For any $\tau \in \mathbb{N}$ and any oblivious adversary, the algorithm $\pr$ satisfies $\;\E[\Switches_\tau(\pr)] \leq 4\sqrt{2 \tau \log n} + 4 \log \tau + 4$.
\end{lemma}
A crude bound thus yields $\E[\Switches_\tau(\pr)] \leq 14 \sqrt{ \tau \log n}$ for all $\tau \in \mathbb{N}$ and $n \geq 2$. So consider applying Framework~\ref{alg:framework} to $\pr$ with $S' =
 322\sqrt{\tfrac{T \log n}{\logtdel}}
$, and call the resulting algorithm Batched Prediction by Random Walk Perturbation ($\bpr_{\delta}$).

\begin{theorem}\label{thm:hp-bpr}
For any $\delta \in (0, \half)$ and any oblivious adversary,
\[
\Prob\left(
\Regret_T(\bpr_{\delta}),\; \Switches_T(\bpr_{\delta}) \leq O\left(\sqrt{T\log n\logdel} \right)
\right)
\geq 1 - \delta
\]
\end{theorem}
\begin{proof}
By Lemma~\ref{lem:switching-bb}, the event $A := \{E < \logtdel\}$ occurs with probability at least $1 - \halfdel$. And whenever this occurs, $\bpr_\delta$ uses at most $322\sqrt{T \log n \logtdel}$ switches. Next we show h.p. guarantees on regret. By Corollary~\ref{cor:regret-fpl-framework},
\begin{align}
\Regret_{T}(\bpr_\delta)
\asleq \Switches_T(\bpr_\delta) + 
\left[\sum_{e \in [E]} \max_{i \in [n]} \sum_{t \in e} P_t(i)\right]
- \sum_{t=1}^{T+1} P_t(i_t)
\label{eq:hp-bpr-regret}
\end{align}
Thus by a union bound with $A$, it suffices to argue that each of these summations is of order $O(\sqrt{T \log n \logdel})$ each with probability at least $1-\tfrac{\delta}{4}$. 
\par \textbf{Bounding the first sum in~\eqref{eq:hp-bpr-regret}.} We argue separately about the expectation and tails of this sum, which we denote by $Y$ for shorthand. Let us first bound the tails of $Y$. To this end, condition on any realization of epoch lengths $\{L_e\}_{e \in [E]}$ summing up to $T$. Note that the resulting conditional distribution of each $P_t(i)$ is clearly still supported within $[-\half, \half]$, and thus is $\half$ sub-Gaussian by Hoeffding's Lemma. Therefore $\sum_{t \in e} P_t(i)$ is $\tfrac{L_e}{2}$ sub-Gaussian for each epoch $e$, and so by the Borell-TIS inequality we have that $\max_{i \in [n]} \sum_{t \in e} P_t(i)$ has $\tfrac{L_e}{2}$ sub-Gaussian tails over its mean $\E [\max_{i \in [n]} \sum_{t \in e} P_t(i)]$. We conclude that $Y = \sum_{e \in [E]} \max_{i \in [n]} \sum_{t \in e} P_t(i)$ has $\tfrac{\sum_{e \in [E]} L_e}{2} = \tfrac{T}{2}$ sub-Gaussian tails over its mean. Therefore, with probability at least $1 - \tfrac{\delta}{4}$, $Y$ is bounded above by $\E[Y] + O(\sqrt{T\logdel})$.
\par We now show how to bound the expectation $\E[Y]$ of this sum. Let us break each epoch into sub-epochs of length at most $L := \tfrac{T}{\logdel}$, and denote the resulting (random) collection of sub-epochs by $E'$. By a simple averaging argument, this increases the number of epochs by at most $\logdel$; that is, $E' \leq E + \logdel$ a.s. holds. It follows by Jensen's inequality that
\begin{align}
Y = \sum_{e \in [E]} \max_{i \in [n]} \sum_{t \in e} P_t(i)
\overset{\text{a.s.}}{\leq} \sum_{e' \in [E']} \max_{i \in [n]} \sum_{t \in e'} P_t(i) 
\label{eq:hp-bpr-sum-1}
\end{align}
Now, for each epoch $e' \in [E']$, we have by construction of $E'$ that the length of $e'$ is at most $|e'| \leq L$. While the length of an epoch may be dependent on the random variables $P_t(i)$ in the epoch, we can still bound for an epoch $e'$ starting at time $t_{e'}$
\begin{align}
\max_{i \in [n]} \sum_{t \in e'} P_t(i) &\overset{\text{a.s.}}{\leq} \max_{i \in [n]} \max_{\tau \in [0,L]} \sum_{t=t_{e'}}^{t_{e'}+\tau} P_t(i)
\label{eq:hp-bpr-sum-2}
\end{align}
This inequality allows us to break the dependence between the epoch length and the variables $P_t$, at the cost of having to bound the max over $\tau \in [0, L]$. Now, a sum such as $\max_{\tau \in [0,L]} \sum_{t=t_{e'}}^{t_{e'}+\tau} P_t(i)$ can be handled by standard martingale tail inequalities. Indeed, let $S_j(i) = \sum_{t = t_{e'}}^{t_{e'}+j} P_t(i)$. For any positive integer $c$, define
\[
S^{(c)}_j(i) := \left\{ \begin{array}{ll}
c &\mbox{if } S_{j'}(i) = c \mbox{ for some } j' \leq j\\
S_j(i) &\mbox{otherwise}
\end{array}\right.
\]
In words, $S^{(c)}_j(i)$ tracks $S_j(i)$ unless it hits $c$ at some point, in which case it thereafter remains constant. It is immediate that $S^{(c)}_j(i)$ is a supermartingale and thus by Azuma-Hoeffding's inequality,
\begin{align*}
\Prob\left(\max_{\tau \in [t_{e'}, t_{e'}+L]} \sum_{t=t_{e'}}^{\tau} P_t(i) \geq c\right) 
=
\Prob\left(\max_{j \in [L]}S_j(i) \geq c\right) &= \Prob\left(S^{(c)}_L(i) \geq c\right)\leq \exp\left(-\frac{c^2}{2L}\right)
\end{align*}
Since we have independence between the different $i \in [n]$, a standard calculation of integrating sub-Gaussian upper-tails yields that
$\E[\max_{i \in [n]} \max_{\tau \in [0,L]} \sum_{t=t_{e'}}^{t_{e'}+\tau} P_t(i)] \leq O(\sqrt{L\log n})$. Combining this with~\eqref{eq:hp-bpr-sum-1} and~\eqref{eq:hp-bpr-sum-2}, we conclude that  $\E[Y] \leq O(\sqrt{L \log n}) \cdot (\E[|E|] + \logdel)$.
%

It remains to bound the expectation of $|E|$. Note that Lemma~\ref{lem:switching-helper} implies that $\sum_{e=1}^k L_e - \tfrac{7\Ltarg}{8}k$ is a submartingale until the stopping condition $\sum_{e=1}^k L_e \geq T - \Ltarg$. By the Optional Stopping Theorem, the expected stopping time $\sigma$ satisfies $0 \leq \E[\sum_{e=1}^{\sigma} L_e] - \tfrac{7\Ltarg}{8}\E[\sigma] \leq T - \tfrac{7\Ltarg}{8}\E[ \sigma]$. Rearranging yields that the expected stopping time is bounded above by $\E[\sigma] \leq \tfrac{8T}{7\Ltarg}$. Finally, by another application of Lemma~\ref{lem:switching-helper}, the last hop of length $\Ltarg$, conditioned on any past, is completed in an expected $O(1)$ additional steps. We conclude that
\begin{align}
\E\left[|E|\right] \leq \tfrac{8T}{7\Ltarg} + O(1) = \tfrac{\logtdel}{7} + O(1) = O(\logdel)
\label{eq:hp-bpr-expected-E}
\end{align}

\par \textbf{Bounding the second sum in~\eqref{eq:hp-bpr-regret}.} At first glance, this appears difficult to bound because of the dependencies that arise since $i_t$ is chosen (partially) based on $P_t$. However one can use the decomposition trick from~\citep{DevLugNeu15} to write the sum as
\begin{align}
- \sum_{t=1}^{T+1} P_t(i_{t-1}) + \sum_{t=1}^{T+1} \Big( P_t(i_{t-1})
- P_t(i_t) \Big)
\label{eq:bpr-regret-decomp}
\end{align}
The first sum is now easily bounded since $i_{t-1}$ and $P_t$ are stochastically independent. In particular, Hoeffding's inequality shows that the first sum in~\eqref{eq:bpr-regret-decomp} is of order $O(\sqrt{T\logdel})$ with probability at least $1 - \tfrac{\delta}{4}$. The second sum in~\eqref{eq:bpr-regret-decomp} is easily bounded using the triangle inequality and the fact that the perturbations $P_t(i)$ are bounded within $\{\plusminus \half \}$
\[
\sum_{t=1}^{T+1} \Big( P_t(i_{t-1})
- P_t(i_t) \Big)
\leq 
\sum_{t=1}^{T+1} \left| P_t(i_{t-1})
- P_t(i_t) \right|
\leq \sum_{t=1}^{T+1} \mathbf{1}(i_t \neq i_{t-1}) = \Switches_T(\bpr_{\delta})
\]
which is of the desired order whenever $A$ occurs.
\end{proof}

\subsection{High probability algorithm for online combinatorial optimization}\label{subsec:hp-oco}
In online linear optimization and online combinatorial optimization, there is often an exponential number of experts but low-dimensional structure between them (see e.g.~\citep{hedging-structured-concepts,AudBubLug13}). As such, na\"ively using a standard PFE algorithm is of course (exponentially) inefficient. Various intricate tricks have been developed to deal with this; the point of this subsection is that our Framework~\ref{alg:framework} also applies easily to these without modification. As an example, we detail how to modify~\citep{DevLugNeu15}'s Online Combinatorial Optimization version of $\pr$. But for brevity of the main text, this is deferred to Appendix~\ref{app:hp-oco}.


\section{Switching-budget PFE}\label{sec:budgets-pfe}
In this section, we characterize the complexity of the switching-budget PFE problem, for all ranges of the switching budget $S \in [T]$. Interestingly, the optimal regret exhibits the following (coarse) phase transition at switching budget size $S = \Theta(\sqrt{T \log n})$. We thus separate the cases into a \textit{high-switching regime} ($S = \Omega(\sqrt{T \log n})$) and \textit{low-switching regime} ($S = O(\sqrt{T\log n})$). 
\begin{theorem}[High-switching regime]\label{thm:budget-fullinfo-highswitching}
When $S = \Omega(\sqrt{T \log n})$, the optimal rate for $S$-switching budget PFE against an oblivious adversary is $\min\left(T, \tilde{\Theta}(\sqrt{T\log n})\right)$. Specifically,
\begin{itemize}
\item \emph{Expectation upper bound:} There exists an efficient $S$-budget algorithm achieving $O\left( \sqrt{T \log n} \log T\right)$ expected regret.
\item \emph{H.p. upper bound:} For any $\delta \in \left(0, \half\right)$, there exists an efficient $S$-budget algorithm
achieving $O\left(\sqrt{T \log n} \logdel \right)$ regret with probability at least $1 - \delta$.
\item \emph{Expectation lower bound:} There exists an oblivious adversary that forces any $S$-budget algorithm to incur expected regret at least $\min\left(T, \Omega(\sqrt{T\log n})\right)$.
\item \emph{H.p. lower bound:} For any $\delta \in \left(0, \half\right)$, there exists an oblivious adversary that forces any $S$-budget algorithm to incur regret $\min(T, \Omega(\sqrt{T \log \tfrac{n}{\delta}}))$ with probability at least $\delta$.
\end{itemize}
\end{theorem}


\begin{theorem}[Low-switching regime]\label{thm:budget-fullinfo-lowswitching}
When $S = O(\sqrt{T \log n})$, the optimal rate for $S$-switching budget PFE against an oblivious adversary is $\min\left(T, \tilde{\Theta}\left(\frac{T \log n}{S}\right)\right)$. Specifically,
\begin{itemize}
\item \emph{Expectation upper bound:} There exists an efficient $S$-budget algorithm achieving $O\left( \frac{T \log n\log T}{S} \right)$ expected regret.
\item \emph{H.p. upper bound:} For any $\delta \in \left(0, \half\right)$, there exists an efficient $S$-budget algorithm achieving $O\left(\frac{T\log n\logdel}{S} \right)$ regret with probability at least $1 - \delta$.
\item \emph{Expectation lower bound:} There exists an oblivious adversary that forces any $S$-budget algorithm to incur expected regret at least $\min\left(T, \Omega\left(\frac{T \log n}{S}\right)\right)$.
\item \emph{H.p. lower bound:} For any $\delta \in \left(0, \half\right)$, there exists an oblivious adversary that forces any $S$-budget algorithm to incur regret $\min(T, \Omega(\frac{T(\log n + \sqrt{\log 1/\delta})}{S}))$ with probability at least $\delta$.
\end{itemize}
\end{theorem}

Note that the extra $\log T$ factor for expected regret in both the above theorems is from na\"ively integrating out the tail of the h.p. algorithms. Removing this log factor is an open question.
\par We first present the proof for the high-switching regime, since it is direct given the machinery we developed above in Section~\ref{sec:hp}.
\\ \begin{proof}[Proof of Theorem~\ref{thm:budget-fullinfo-highswitching}]
We first prove the h.p. guarantee. Mini-batch the $T$ iterations into $T' = O(T/\logdel)$ epochs each of uniform size $O(\logdel)$. Applying the h.p. algorithm from Theorem~\ref{thm:hp-sc} to this batched game thus yields an algorithm $\calA$ that with probability at least $1 - \delta$, incurs at most $O(\logdel \cdot \sqrt{T' \log n \logdel}) = O(\sqrt{T \log n} \logdel)$ regret and makes at most $O(\sqrt{T' \log n \logdel}) = O(\sqrt{T \log n})$ switches. Define now the algorithm $\calA'$ that runs $\calA$ but if it ever  exhausts $S$ switches, then it just stays on the same action for the rest of the game. By construction, $\calA'$ deterministically never uses more than $S$ switches, and is thus an $S$-budget algorithm. Moreover, with probability at least $1 - \delta$, $\calA$ (with an appropriate choice of constant in the above mini-batching argument) also uses no more than $S$ switches, in which event the actions of $\calA$ and $\calA'$ are identical, and so in particular $\Regret_T(\calA') = O(\sqrt{T\log n} \logdel)$.
\par The expectation guarantee is proved by invoking this h.p. guarantee with $\delta = \tfrac{1}{T}$, expanding $\E[\Regret]$ by conditioning on the (regret) failure event of this algorithm, and using the trivial observation that regret is always upper bounded by $T$. The expectation lower bound follows immediately from the classical $\Omega(\sqrt{T \log n})$ lower bound for PFE \textit{without} constraints on switching~\citep{CB-expert}. The h.p. lower bound follows by analyzing the tails for the same adversary and is deferred to Appendix~\ref{app:subsec:lb-regret}.

\end{proof}

We now shift our attention to the low-switching regime. Both the upper and lower bounds will be proved using mini-batching reductions. The upper bounds (algorithms) are straightforward so we present them first.
\\ \begin{proof}[Proof of achievability in Theorem~\ref{thm:budget-fullinfo-lowswitching}] WLOG we may restrict to $S = \omega(\sqrt{\log n})$ since otherwise the statement is trivially satisfied. For the h.p. algorithm, minibatch into $T' := \tfrac{S^2}{\log n}$ epochs (so that $S = \sqrt{T' \log n}$) and apply the h.p. algorithm from Theorem~\ref{thm:budget-fullinfo-highswitching}. The expectation guarantee follows from an identical argument as in Theorem~\ref{thm:budget-fullinfo-highswitching}.
\end{proof}
\textbf{Proof sketch of lower bound in Theorem~\ref{thm:budget-fullinfo-lowswitching}.} (Full details in Appendix~\ref{app:fullinfo-lowerbound}.) The idea is essentially a batched version of~\cite{CB-expert}'s classical lower bound for unconstrained PFE. So let us first recall that argument. That construction draws the loss of each expert in each iteration i.i.d. from $\{0,1\}$ uniformly at random. A simple argument shows any algorithm has expected loss $\tfrac{T}{2}$, but that the best expert has loss concentrating around $\tfrac{T}{2} - \Theta(\sqrt{T \log n})$ since (after translation by $\tfrac{T}{2}$) it is the minimum of $n$ i.i.d. simple random walks of length $T$. Therefore they conclude $\E[\Regret] = \Omega(\sqrt{T \log n})$.
\par However, that adversarial construction does not capitalize on the algorithm's limited switching budget in our setting. We accomplish this by increasing the variance of the random walk in a certain way that a switch-limited algorithm cannot benefit from. Specifically, proceed again by batching the $T$ iterations into roughly $E \approx \tfrac{S^2}{\log n}$ epochs, each of uniform length $\tfrac{T}{E}$. For each epoch and each expert, draw a single $\Ber(\half)$ and assign it as that expert's loss for each iteration in that epoch.
\par Informally, the optimal algorithm still incurs expected loss of half for each iteration in epochs it does not switch in; and loss of $0$ for each epoch it switches in. Critically, however, the algorithm can switch at most $S$ times, which is small compared to the number of epochs $E$. Thus any algorithm incurs expected loss roughly $\approx \tfrac{T}{E}\left(\tfrac{E}{2} - S\right) = \tfrac{T}{2} - \Theta\left(\tfrac{T\log n}{S}\right)$. Moreover, the best expert now has loss concentrating around $\frac{T}{E}\left(\tfrac{E}{2} - \Theta(\sqrt{E\log n})\right) = \tfrac{T}{2} - \Theta\left(\tfrac{T\log n}{S} \right)$.
\par Therefore, after appropriately choosing constants in the epoch size, we can then conclude that the regret of any $S$-budget algorithm is $\Omega\left(\tfrac{T\log n}{S}\right)$ in expectation, with sub-Gaussian tails of size $\Omega(\tfrac{T(\log n + \sqrt{\log 1/\delta})}{S} )$. Full details deferred to Appendix~\ref{app:fullinfo-lowerbound}.


\section{Switching-budget MAB}\label{sec:budgets-mab}
In this section, we characterize the complexity of switching-budget MAB, for all ranges of the switching budget $S \in [T]$. Interestingly, there is no phase transition here: the regret smoothly decays as a function of the switching budget. 
\begin{theorem} \label{thm:mab}
The optimal rate for $S$-switching-budget MAB against an oblivious adversary is $\min\left(\tilde{\Theta}\left(\frac{T\sqrt{n}}{\sqrt{S}}\right), T \right)$. Specifically,
\begin{itemize}
\item \emph{Expectation upper bound:} There exists an efficient $S$-budget algorithm achieving $O\left(
\frac{T\sqrt{n}}{\sqrt{S}}
\right)$ expected regret.
\item \emph{H.p. upper bound:} There exists an efficient $S$-budget algorithm that for all $\delta \in (0,1)$ achieves $O\left(\frac{T\sqrt{n}}{\sqrt{S}}\frac{\log\left(n/\delta\right)}{\sqrt{\log n}}\right)$ regret with probability at least $1 - \delta$.
\item \emph{Expectation lower bound:} There exists an oblivious adversary that forces any $S$-budget algorithm to incur expected regret at least $\min\left(T, \Omega\left(\frac{T\sqrt{n}}{\sqrt{S}\log^{3/2}T} \right)\right)$.
\end{itemize}
\end{theorem}
The upper bound proofs are just standard mini-batching arguments (see e.g.~\citep{AroDekTew12}).
\\ \begin{proof}[Proof of upper bounds in Theorem~\ref{thm:mab}]
Mini-batch the $T$ iterations into $S$ epochs of $\tfrac T S$ consecutive iterations. After a rescaling of the epoch losses by $\tfrac S T$, this becomes an unconstrained MAB problem. Therefore applying the results in~\citep{AudBub10} gives the desired expected and h.p. guarantees on regret.
\end{proof}

The proof of the lower bound is more involved. We present two ways of proving this result, both of which appeal to results developed in the elegant work of~\citep{DekDinKorPer}, which gave the first tight (up to a logarithmic factor in $T$) lower bound for switching-cost MAB.
\par The first proof is elementary and quick, but does not give an explicit adversarial construction: we prove the desired switching-budget MAB lower bound via a reduction to the switching-cost MAB lower bound of~\citep{DekDinKorPer}. This approach is inspired by the proof of Theorem 4 in~\citep{DekDinKorPer}, which uses a similar type of argument. However their reduction obtains the wrong dependence on the number of actions $n$ and also loses track of polylogarithmic factors. In Subsection~\ref{subsec:mab-lb-reduction}, we give a more careful reduction that fixes these issues.
\par The second proof is significantly more complicated, but direct and also gives an explicit adversarial construction: we use a modification of the multi-scale random walk construction developed in~\citep{DekDinKorPer}. Specifically, we show that the constant gap $\eps$ between the best action and all other actions can be enlarged to roughly 
$\sqrt{\tfrac{n}{S}}$
while ensuring that no $S$-budget algorithm can information-theoretically distinguish the best action. Informally, this implies that any $S$-budget algorithm incurs expected regret of order roughly 
$T\eps = \tfrac{T\sqrt{n}}{\sqrt{S}}$
as desired. We present this proof since it gives an explicit adversarial construction, whereas the first proof does not; however since the analysis is essentially identical to that in~\citep{DekDinKorPer}, the proof is deferred to Appendix~\ref{app:mab-lowerbound-construction}. 

\subsection{Proof 1 of lower bound in Theorem~\ref{thm:mab}: via reduction to switching-cost MAB}\label{subsec:mab-lb-reduction}
Denote the minimax rates for the switching-budget and switching-cost settings, respectively, by
\begin{align*}
R(T,n,S) &:= \min_{S\text{-budget alg}} \max_{\text{oblivious adversary}} \E[\Regret] \\
R(T,n,c) &:= \min_{\text{alg}} \max_{\text{oblivious adversary}} \E[\Regret \;+\; c \cdot \Switches]
\end{align*}
\begin{proof}[Proof of lower bound in Theorem~\ref{thm:mab}]
Observe we may WLOG restrict to $S > \max\left(\tfrac{n}{2}, 3\right)$. This is sufficient since when $S \leq \tfrac{n}{2}$, the algorithm cannot try all arms and thus cannot achieve sublinear regret\footnote{The following construction makes this formal. The adversary selects one good arm $i^*$ uniformly at random, and then defines losses $\ell_t(i) = 1(i = i^*)$. Since the algorithm receives only bandit feedback and can play at most $\tfrac{n}{2}$ actions, thus clearly with probability at least $\half$ it never plays the best action, therefore $\E[\Regret] \geq \tfrac{T}{2}$.}. And when $S \leq 3$, then by monotonicity of $R$ in its last argument, $R(T,n,S) \geq R(T,n,4)$ so the desired bound follows for free up to a constant factor of at most $\sqrt{4} = 2$.
\par Observe that for any $S$, we can apply an optimal $S$-budget algorithm (i.e. achieving $\E[\Regret] = R(T,n,S)$), to achieve an expected cost of at most $R(T,n,S) + cS$ for the $c$-switching-cost problem. Said more concisely, this implies that for any $c > 0$,
\begin{align}
R(T,n,c) \leq 
\min_{S \in [T]} \Big[R(T,n,S) + cS \Big] 
\label{eq:duality-easy-direction}
\end{align}
Now Theorem 3 of~\citep{DekDinKorPer} shows that for any $c \in \left(0, \frac{T}{\max(n, 6)}\right)$,
\[
R(T,n,c)
\geq \frac{T^{2/3}n^{1/3}c^{1/3}}{50 \log T}
\]
So fix any $S \in [T]$ such that $S > \max\left(\tfrac{n}{2}, 3\right)$, and define $c := \tfrac{R(T,n,S)}{2S}$. Then $c \in \left(0, \tfrac{T}{\max(n, 6)}\right)$ and so combining the above two displays yields 
\begin{align*}
R(T,n,S)
\geq \frac{T^{2/3}n^{1/3}c^{1/3}}{100 \log T}
\end{align*}
Use the definition of $c$ and then invoke the above display to conclude
\begin{align*}
\frac{1}{\sqrt{2}\left(100 \log T\right)^{3/2}} \left(\frac{T\sqrt{n}}{\sqrt{S}}\right)
&= 
\frac{1}{\sqrt{2}\left(100 \log T\right)^{3/2}} \left(
\frac{T\sqrt{n}}{\sqrt{\left(\frac{R(T,n,S)}{2c}\right)}}\right)
\leq \frac{R(T,n,S)^{3/2}}{R(T,n,S)^{1/2}}
= R(T,n,S)
\end{align*}
\end{proof}


\section{Duality between switching-cost and switching-budget settings}\label{sec:duality}
One direction of this duality is simple and folklore: as mentioned in the discussion around~\eqref{eq:duality-easy-direction} above,
any $S$-switching-budget algorithm with expected regret upper bounded by $R(T,n,S)$ clearly yields a $c$-switching-cost algorithm with expected cost at most $R(T,n,c) \leq R(T,n,S) + cS$.
Note that plugging into~\eqref{eq:duality-easy-direction} the bounds on $R(T,n,S)$ proved in this paper, immediately recovers the corresponding known upper bounds on $R(T,n,c)$ up to polylogarithmic factors in $T$, for both PFE and MAB. Indeed, setting $S_{\pfe}(c) = \tilde{\Theta}(\sqrt{\tfrac{T\log n}{c}})$ yields $R_{\pfe}(T,n,c) \leq \tilde{O}(\sqrt{c T \log n} )$; and setting $S_{\mab}(c) = \tilde{\Theta}(\tfrac{T^{2/3}n^{1/3}}{c^{2/3}})$ yields $R_{\mab}(T,n,c) \leq \tilde{O}(\tfrac{T\sqrt{n}}{S^{3/2}} )$.
\par However, the other direction of this duality is not obvious, since a $c$-switching cost algorithm with expected cost at most $R(T,n,c)$ might be unusable for the switching-budget setting. Indeed when black-boxed, this algorithm yields only the upper bound of $c^{-1}R(T,n,c)$ switches \textit{in expectation}, as opposed to the hard-cap requirement needed for the switching-budget setting. Of course, the mini-batching MAB algorithm from~\citep{AroDekTew12} by construction achieves this hard-cap requirement deterministically, so this converse duality direction is easy for MAB. But for PFE this direction is not clear, since existing algorithm's upper tails on switching are too large to be applicable for the switching-budget setting (see discussions in Section~\ref{subsec:previous-work} and Appendix~\ref{app:lb-kv-sd}). 
\par One way of interpreting the results from Sections~\ref{sec:hp}, is that (almost) nothing is lost by requiring $c$-switching cost algorithms to have h.p. bounds on switching. Such algorithms can then certainly be applied to the switching-budget problem, by setting the switching-failure probability $\delta \approx \tfrac{1}{T}$ and recalling that regret is always bounded by $T$. This yields the desired other direction of the duality (up to a single logarithmic factor in $T$).
\par This discussion is summarized formally by the proceeding remark.

\begin{remark}\label{rem:duality}
The complexity of the $c$-switching-cost and $S$-switching budget setting are equivalent (in terms of minimax rates being equal up to a polylogarithmic factor in $T$) when:
\begin{itemize}
\item PFE\footnote{When $S = \tilde{\Omega}(\log n)$ or equivalently $c = \tilde{O}\left(\tfrac{T}{\log n}\right)$, since otherwise the minimax rate is uninterestingly $\tilde{\Omega}(T)$.}: $S = \tilde{\Theta}\left(\tfrac{\sqrt{T\log n}}{\max(c, 1)}\right)$ or equivalently $c = \tilde{\Theta}\left(\max\left(\tfrac{T\log n}{S^2}, 1\right)\right)$.
\item MAB\footnote{When $S = \tilde{\Omega}(n)$, or equivalently $c \leq \tilde{O}\left(\tfrac{T}{n}\right)$, since otherwise the minimax rate is uninterestingly $\tilde{\Omega}(T)$.}: $S = \tilde{\Theta}\left(\tfrac{T^{2/3}N^{1/3}}{c^{2/3}}\right)$ or equivalently $c = \tilde{\Theta}\left(\tfrac{T\sqrt{N}}{S^{3/2}}\right)$.
\end{itemize}
\end{remark}

One can visualize this duality as follows. Consider (for any $c \geq \frac{1}{T}$), the unique point $P = (P_x,P_y)$ at the intersection of the line $y = cx$ with the complexity profile in Figure~\ref{fig:plots} (Figure~\ref{fig:pfe} for PFE or~\ref{fig:mab} for MAB). Then, up to a small polylogarithmic factor in $T$, $P_y$ is equal to the minimax rate for both the $c$-switching-cost setting and the $S := P_x$-switching-budget setting; and moreover $c$ and $S$ are related by the duality formulas in Remark~\ref{rem:duality} above.
\section{Conclusions}

In this work, we studied online learning over a finite action set, in the presence of a budget for switching. While this problem is closely related to the switching cost setting, handling switching budgets requires obtaining high probability bounds on the number of switches. We presented a general approach for converting FPL-type algorithms into algorithms with high probability bounds on switches as well as regret. Using this result, we showed tradeoffs between the regret and the switching budget that are tight up to logarithmic factors.

We conclude with some open questions. The most natural open question is to close the polylogarithmic gaps between the upper and lower bounds for the regret in the presence of switching budgets, for both the experts and the bandit setting. Another natural question is to give a uniform high probability algorithm, i.e. a single algorithm for PFE that yields bounds similar to Theorem~\ref{thm:hp-sc} simultaneously for all $\delta$.

\paragraph*{Acknowledgements.} We are thankful to the three anonymous COLT 2018 reviewers for their helpful comments. We are indebted to Elad Hazan for numerous fruitful discussions and for suggesting the switching-budget setting to us. We also thank Yoram Singer, Tomer Koren, David Martins, Vianney Perchet, and Jonathan Weed for helpful discussions.
\par Part of this work was done while JA was visiting the Simons Institute for the Theory of Computing, which was partially supported by the DIMACS/Simons Collaboration on Bridging Continuous and Discrete Optimization through NSF grant \#CCF-1740425. JA is also supported by NSF Graduate Research Fellowship 1122374.

\newpage
\addcontentsline{toc}{section}{References}
\bibliography{bibSwitching}

\newpage
\appendix

\section{Adaptive adversaries in the switching-budget setting}\label{app:adaptive}
In this section, we make formal the notion that \textit{adaptive} adversaries are too powerful to compete against in the switching-budget setting. This is why this paper focuses on \textit{oblivious} adversaries.
\par The lower-bound construction is quite simple, and is identical to the folklore construction for adaptive adversaries in the switching-cost setting.
\begin{theorem}
There is a deterministic adaptive adversary for PFE with $n=2$ actions, that forces any $S$-switching-budget algorithm $\calA$ to incur regret at least
\[
\Regret(\calA) \geq \frac{T-1}{2} - S
\]
In particular, for any $S = o(T)$, the minimax rate for regret is $\Theta(T)$.
\end{theorem}
\begin{proof}
Define the adaptive adversary which constructs losses as follows
\begin{align*}
\ell_{t}(i) :=
\begin{cases}
0 & t = 1 \\
\mathbf{1}\{i = i_{t-1} \} & t > 1
\end{cases}
\end{align*}
where $i_t$ is the action $\calA$ plays at iteration $t$. First, observe that these losses force $\calA$ to incur cumulative loss at least $T-S - 1$. This is because $\calA$ incurs a loss of $1$ whenever it does not switch, and a loss of $0$ whenever it does (but it can only do switch at most $S$ times). Formally:
\[
\sum_{t=1}^T \ell_t(i_t)
= \sum_{t=2}^T \mathbf{1}\{i_t = i_{t-1} \}
= (T - 1) - \sum_{t=2}^T \mathbf{1}\{i_t \neq i_{t-1} \}
\geq T - S - 1
\]
On the other hand, the best action has cumulative loss at most $\frac{T-1}{2}$ by a simple averaging argument. Plugging in the definition of regret finishes the proof.
\end{proof}
\section{Our framework needs more than just expectation bounds}
\label{app:need_properties}
Our analysis of Algorithm~\ref{alg:framework} needs specific properties of the algorithm $\calA$ that it builds on.
We next argue that our approach cannot convert an arbitrary algorithm $\calA$ achieving guarantees in expectation into one achieving them w.h.p. Indeed, consider an arbitrary algorithm $\calA$ satisfying the property that $\E[\Regret_T(\calA)]$ and $\E[\Switches_T(\calA)]$ are both $O(\sqrt{T\log n})$, and define a new algorithm $\calA_p$ by: with probability $1-p$ run $\calA$ for all $T$ iterations; otherwise with probability $p$ play action $\bot$ for $\sqrt{T \log n}/p$ iterations, and then run $\calA$ for the remainder. Here $\bot$ is a new action that incurs loss $1$ in each step. Clearly $\calA_p$ also satisfies $\E[\Regret_T(\calA_p)], \E[\Switches_T(\calA_p)] = O(\sqrt{T\log n})$. However, applying the meta-framework in Figure~\ref{alg:framework} to $\calA_p$ cannot produce an algorithm with the desired h.p. guarantees. Indeed, when $p \approx \delta/\logdel$, then with probability roughly $1 - (1 - p)^{N} \approx \delta$, we encounter the bad event for $\calA_p$ in at least one of the epochs. When this occurs, we incur regret roughly $\sqrt{T \log n}/p \approx \sqrt{T\log n}/\delta$, instead of the $\sqrt{T\log n \logdel}$ required for a h.p. guarantee.
\


\section{Proof of standard lemma in analysis of $\fpl$-type algorithms: Lemma~\ref{lem:regret-fpl-framework}}\label{app:betheleader}
In this section, we present for completeness a standard proof of Lemma~\ref{lem:regret-fpl-framework}~\citep{KalVem, Book-CB-Lugosi, DevLugNeu15}. The key step in its proof is to compare it to the Be-The-Leader algorithm ($\btl$), which is known to have negative regret.
\par Formally, $\btl$ plays at iteration $t$ the action $i_{t+1} := \argmin_{i \in [n]} \sum_{s=0}^{t} \ell_t(i)$.\footnote{Note that although $\btl$ is well-defined, it is of course not a ``legitimate'' online learning algorithm since we do not have access to the loss $\ell_t$ at iteration $t$. As such, $\btl$ is only for analysis purposes.} In other words, it plays the action that the Follow the Leader algorithm ($\ftl$) would play at iteration $t+1$. The following so-called ``Be-The-Leader'' lemma shows that $\btl$ has negative regret. It has a one line induction proof~\citep{KalVem}.
\begin{lemma}[Be-The-Leader lemma,~\citep{KalVem}]\label{lem:btl} For all $i \in [n]$,
\[
\sum_{t=0}^T \ell_t(i_{t+1})
\leq 
\sum_{t=0}^T \ell_t(i)
\]
\end{lemma}

\noindent The proof of Lemma~\ref{lem:regret-fpl-framework} now follows readily from Lemma~\ref{lem:btl}.
\\ \begin{proof}[Proof of Lemma~\ref{lem:regret-fpl-framework}]
Let $i^* := \argmin_{i \in [n]} \sum_{t=1}^T \ell_t(i)$ be the best action in hindsight (w.r.t. the true losses). Applying Lemma~\ref{lem:btl} to the regularized losses $\{\hatl_t(i) = \ell_t(i) + P_{t+1}(i)\}_{t \in \{0, \dots, T\}, i \in [n]}$,
\[
\sum_{t=0}^T \hatl_t(i_{t+1})
\leq 
\sum_{t=0}^T \hatl_t(i^*)
\]
Using the definition of $\hatl_t(i)$ and the fact that $\ell_0(i) = 0$, we can expand the LHS as
\[
\sum_{t=0}^T \hatl_t(i_{t+1})
= 
\sum_{t=1}^T \ell_t(i_{t}) + \sum_{t=1}^{T+1} P_t(i_t) + 
\sum_{t=1}^T \Big(\ell_t(i_{t+1}) - \ell_t(i_{t}) \Big)
\]
The RHS can similarly be expanded as
\[
\sum_{t=0}^T \hatl_t(i^*)
= 
\sum_{t=1}^T \ell_t(i^*) + \sum_{t=1}^{T+1}P_t(i^*)
\]
Combining the above two displays and using the definition of regret gives
\[
\Regret_T(\fpl)
= \sum_{t=1}^T \ell_t(i_t) - \sum_{t=1}^T \ell_t(i^*)
\leq 
\sum_{t=1}^{T+1}P_t(i^*)
- 
\sum_{t=1}^{T+1} P_t(i_t)
+
\sum_{t=1}^T \Big(\ell_t(i_{t}) - \ell_t(i_{t+1}) \Big)
\]
The proof is complete by the trivial observation $\sum_{t=1}^T P_t(i^*) \leq \max_{i \in [n]} \sum_{t=1}^T P_t(i)$ and the bound
\[
\sum_{t=1}^T \Big(\ell_t(i_{t}) - \ell_t(i_{t+1})\Big) 
\leq \sum_{t=1}^T \Big|\ell_t(i_{t}) - \ell_t(i_{t+1})\Big| 
\leq M\sum_{t=1}^T \mathbf{1}(i_t \neq i_{t+1})
= M \cdot \Switches_T(\fpl)
\]
where we have respectively used the triangle inequality, the definition of $M := \sup_{\text{loss }\ell\text{, actions } i,i'} |\ell(i) - \ell(i')|$, and the fact that $\fpl$ by definition plays the regularized leaders. 
\end{proof}

\section{Proof of concentration inequalities for $\fpl$ regularization: Lemma~\ref{lem:PEV}}\label{app:chernoff}
Recall we wish to show the concentration inequality
\[
\Prob\left(
\sum_{e=1}^N X_e
> 6N\log n
\right)
< e^{-N}
\]
where $N, n \geq 2$, $\{R_e(i)\}_{e \in [N], i \in [n]}$ are i.i.d. standard exponentials, and $X_e := \max_{i \in [n]} R_e(i)$. We will prove this via Chernoff bounds. As such, the first step is to bound the MGF of each $X_e$, i.e. bound the MGF of the maximum order statistic of $n$ independent standard exponentials.
\par Note that our MGF bounds are rather crude (we use a union bound to upper bound the tail distribution of a maximum order statistic); nevertheless this gives the optimal concentration rate up to a constant factor. Indeed, since the expectation of the maximum of $N$ i.i.d. standard exponentials is 
$\sum_{i=1}^n i^{-1} \in [\log n, \log n + 1] \geq \log n$, thus $\E\left[\sum_{e=1}^N X_e\right] \geq N \log n$, which is already within a factor of $6$ of the concentration inequality we show in Lemma~\ref{lem:PEV}.

\begin{lemma}\label{lem:mgf}
Let $X$ be the maximum of $n$ i.i.d. standard exponential random variables. 
Then for all $t \in (0,1)$,
\[
\E\left[e^{tX}
\right]
\leq \frac{n^t}{1 - t}
\]
\end{lemma}
\begin{proof}
Recall that for a positive r.v. $Y$, $\E[Y] = \int_0^{\infty} \Prob(Y \geq u) du$. Thus
\begin{align}
\E\left[e^{tX}
\right]
= \int_0^{n^t} \Prob\left(e^{tX} \geq u\right) du
+ \int_{n^t}^{\infty} \Prob\left(e^{tX} \geq u\right) du\label{eq:mgf-split}
\end{align}
The first integral is trivially upper bounded by $n^t$. For the second integral, perform a change of variables $u = n^{t(1 + \delta)}$. Then $du = n^{t(1+\delta)} t\log n d\delta$ so
\begin{align*}
\int_{n^t}^{\infty} \Prob\left(e^{tX} \geq u\right) du
&\leq tn^t \log n
\int_{0}^{\infty} \Prob\left(X \geq (1 + \delta) \log n \right) n^{\delta t} d\delta
\end{align*}
Now since $X = \max_{i \in [n]} R(i)$ are i.i.d. standard exponentials, we have by a union bound
\begin{align*}
\Prob(X \geq (1 + \delta) \log n)
&= 1 - \Prob(\max_{i \in [n]} R(i) < (1 + \delta) \log n)
\\ &= 1 - \Prob(R(i) < (1 + \delta) \log n)^n
\\ &= 1 - \left[1 - n^{-(1 + \delta)}\right]^n
\\ &\leq n^{-\delta}
\end{align*}
Combining the above two displays gives
\begin{align*}
\int_{n^t}^{\infty} \Prob\left(e^{tX} \geq u\right) du
&\leq tn^t \log n
\int_{0}^{\infty} n^{\delta(t-1)} d\delta
\\ &= tn^t \log n
\left[
\frac{n^{\delta(t-1)}}{(t-1) \log n}
\right]_0^{\infty}
\\ &= \frac{t}{t-1} n^t 
\end{align*}
Plugging this into~\eqref{eq:mgf-split}, we conclude $
\E\left[e^{tX}
\right] \leq n^t\left(1 + \frac{t}{1-t}\right) = \frac{n^t}{1-t}$ as desired.
\end{proof}

The proof of Lemma~\ref{lem:PEV} now follows readily.
\\ \begin{proof}[Proof of Lemma~\ref{lem:PEV}]
By a standard Chernoff argument and then applying Lemma~\ref{lem:mgf} with $t=\half$,
\begin{align*}
\Prob\left(\sum_{e=1}^N X_e \geq a\right)
\leq e^{-ta} \E\left[e^{tX_e}\right]^N
\leq e^{-\half \left(a -N\log n - 2N \log 2\right)}
\end{align*}
The proof is complete by taking $a := N \log n + 2N\log 2 + 2N$, which is crudely bounded above by $6 N \log n$ since by assumption $N,n\geq 2$.
\end{proof}

\section{High probability algorithm for online combinatorial optimization}\label{app:hp-oco}
\par For simplicity, we restrict to the case of online combinatorial optimization over subsets of the binary hypercube, which as e.g.~\citep{combinatorial-bandits} and~\citep{DevLugNeu15} argue, includes most of the important applications. Formally, the experts are elements of a decision set $S \subseteq \{0,1\}^d$, $|S| = n$. Losses are linear functions $\ell_t(v) := \ell_t^Tv$ where $\ell_t \in [0,1]^d$ are constrained in infinity norm. A (non-essential) assumption often made in the literature, that allows for more specific bounds, is that each action $v \in S$ is $m$-sparse, i.e. each $v \in S$ satisfies $\|v\|_0 = \|v\|_1 = m$. (For example, if $S$ is the set of $r \times r$ permutation matrices, then $m$ is only $r$ rather than $d = r^2$.) We refer the reader to the introduction of~\citep{DevLugNeu15} for a summary of the state-of-the-art for this problem.
\par \citep{DevLugNeu15} introduce an efficient algorithm for this setting that we call here Combinatorial $\pr$ ($\cpr$). It also falls under the FPL framework: the perturbations are i.i.d. Gaussian random vectors $P_1, \dots, P_T \sim \mathcal{N}(0,\eta^2I_{d\times d})$, for some parameter $\eta$. At each iteration $t \in [T]$, $\cpr$ then plays the leading action $v_t := \argmin_{v \in S} (\sum_{s=0}^{t-1} \hatl_s )^Tv$
w.r.t. the perturbed losses $\hatl_s := \ell_s + P_{s+1}$. 

\par Following~\citep{DevLugNeuCOLT}, we will set $\eta = \log^{-\frac 1 2} d$, which the authors show leads to regret $O(m^2 \sqrt{T} \log d)$ and $O(m\sqrt{T} \log d)$ switches, both in expectation\footnote{In the journal version~\citep{DevLugNeu15}, a different setting of $\eta$ is used which leads to a slightly better regret bound of $O(m^{\frac 3 2}\sqrt{T \log d})$ (See Theorem 4 of~\citep{DevLugNeuCOLT} and Theorem 1 of~\citep{DevLugNeu15}.) However, to show this they employ the proof technique of~\citep{NeuBar13} which is not amenable to the h.p. framework and analysis developed in Section~\ref{sec:hp}. We use the algorithm and the analysis approach of~\citep[Theorem~4]{DevLugNeuCOLT}.}.
We apply Framework~\ref{alg:framework} to $\cpr$ to obtain an algorithm that we call Batched $\cpr$ ($\bcpr_{\delta}$); we next show that this achieves both of these w.h.p. (up to an extra logarithmic factor in the regret).

\par The analysis is similar to the earlier proofs for $\bmfpl$ and $\bpr_{\delta}$ in Sections~\ref{subsec:hp-mfpl} and~\ref{subsec:hp-pr}, respectively. We start with an upper bound on the expected number of switches. (This is property (i) of Framework~\ref{alg:framework}; see Subsection~\ref{subsec:hp-framework}.)
\begin{lemma}[Lemma 5 of ~\citep{DevLugNeuCOLT}]\label{lem:cpr-switching} For any $\tau \in \mathbb{N}$ and any oblivious adversary,
$\E[\Switches_\tau(\cpr)] \leq \sum_{t=1}^\tau \frac{m\E\|\ell_t + P_{t+1}\|_{\infty}^2}{2\eta^2t} + \frac{m\sqrt{2\log d}\E\|\ell_t + P_{t+1}\|_{\infty}}{\eta\sqrt{t}}$.
\end{lemma}
Standard bounds on suprema of Gaussian processes give that $\E[\|P_1\|_{\infty}] \leq \eta\sqrt{2 \log d}$ and $\E[\|P_1\|_{\infty}^2] \leq \eta^2(2\log d + \sqrt{2 \log d} + 1)$~\citep{Concentration}. Thus
\begin{align*}
	\E[\Switches_{\tau}(\cpr)] &\leq  O\left(m\left(\frac{1 + \eta^2 \log d}{\eta^2}\right)\log \tau\right) + O\left(m\left(\frac{1+\eta\sqrt{\log d}}{\eta}\right)\sqrt{\tau\log d}\right)
\end{align*}
which for $\eta = \log^{-\frac 1 2} d$ is bounded above by $c m\sqrt{\tau} \log d$ for some $c >0$. Denote by Batched $\cpr$ ($\bcpr_{\delta}$) the algorithm produced by Framework~\ref{alg:framework} with $S' := 23c m \sqrt{\tfrac{T}{\logtdel}} \log d$.

\begin{theorem}\label{thm:hp-bcpr}
For any $\delta \in (0, \half)$ and any oblivious adversary,
\begin{equation}
\begin{aligned}
\Prob\Bigg( & \Regret_T(\bcpr_{\delta}) \leq O\left(m^2\sqrt{T \logdel}\log d \cdot \sqrt{\log\left(\tfrac{dT}{\delta}\right)} \right),\\
     & \Switches_T(\bcpr_{\delta}) \leq O\left(m\sqrt{T \logdel}\log d \right) \Bigg) \geq 1 - \delta
\end{aligned}
\nonumber
\end{equation}
\end{theorem}
\begin{proof}
By Lemma~\ref{lem:switching-bb}, the event $A := \{E \leq \logtdel\}$ occurs with probability at least $1 - \tfrac{\delta}{2}$. And whenever this happens, $\bcpr_{\delta}$ switches at most $23c m \log d \sqrt{T \logtdel}$ times.
\par Next we show h.p. guarantees on regret. By H\"older's Inequality, $\sup_{\ell \in [0,1]^d, \; v,v' \in S} |\ell^Tv - \ell^Tv'| \leq m$. Thus by Corollary~\ref{cor:regret-fpl-framework},
\begin{align}
\Regret_{T}(\bcpr_{\delta})
\asleq m \cdot \Switches_T(\bcpr_{\delta}) + 
\left[\sum_{e \in [E]} \max_{v \in S} \left(\sum_{t \in e} P_t\right)^T v\right]
- \sum_{t=1}^{T+1} P_t^Tv_t
\label{eq:proof-hp-bcpr-regret}
\end{align}
So by a union bound with $A$, it suffices to show each of these summations are of the desired order with probability at least $1-\tfrac{\delta}{4}$. We proceed by similar arguments to the ones we used in the analysis of $\bpr$ in Section~\ref{subsec:hp-pr}.

\par \textbf{Bounding the first sum in~\eqref{eq:proof-hp-bcpr-regret}}. This part of the proof is nearly identical to the analogous argument in the proof of Theorem~\ref{thm:hp-bpr}. Denote this sum by $Y$. We first bound its upper tails. Fix any realization of epoch lengths $\{L_e\}_{e \in [E]}$ summing to $T$. Since the $P_{t} \sim \mathcal{N}(0, \eta^2I)$ are i.i.d., thus $Z_e := \sum_{t \in e} P_t$ has law $\mathcal{N}(O, L_e \eta^2 I)$. Thus $X_{v,e} := Z_e^Tv$ is distributed $\mathcal{N}(0, mL_e\eta^2)$ for each $v \in S$. (Of course $X_{v,e}$ are not necessarily independent.) The Borell-TIS inequality gives that $\sup_{v \in S} X_{v,e}$ has $mL_e\eta^2$ sub-Gaussian tails around its mean~\citep{GaussianIneqs}. We conclude that $Y = \sum_{e \in [E]} \sup_{v \in S} X_{v,e}$ has $mT\eta^2$ sub-Gaussian tails around its mean $\E[Y]$.

We next bound this expectation $\E[Y]$. As in the proof of Theorem~\ref{thm:hp-bpr}, we split epochs longer than $L := T\log^{-1}\tfrac{1}{\delta}$ so as to ensure that each sub-epoch is of length at most $L$. By a simple averaging argument, the number of sub-epochs is at most $\logdel$ more than the number of epochs. Denote the set of sub-epochs by $E'$. By Jensen's inequality, it suffices to bound $\sum_{e' \in E'} \sup_{v \in S} \left(\sum_{t \in e'} P_t \right)^T v$. For a sub-epoch starting at $t_{e'}$, we can bound
\begin{align*}
\sup_{v \in S} \left(\sum_{t  \in e'} P_t\right)^T v &\leq \sup_{v \in S} \sup_{\tau \in [0,L]} \left(\sum_{t  = t_{e'}}^{t_{e'} + \tau} P_t\right)^T v.
\end{align*}
This allows us to break the dependence between the epoch length and the variables $P_t$ at the cost of having to bound the max over $\tau \in [0, L]$. For any $v$, a martingale argument identical to that in the proof of Theorem~\ref{thm:hp-bpr} yields the following tail bound on $\sup_{\tau \in [0, L]} \left(\sum_{t  = t_{e'}}^{t_{e'} + \tau} P_t\right)^T v$:
\begin{align*}
\Prob\left(\sup_{\tau \in [0, L]} \left(\sum_{t  = t_{e'}}^{t_{e'} + \tau} P_t\right)^T v > c\eta \sqrt{2mL}\right) &\leq \exp(-c^2)
\end{align*}
Using a union bound on ${d \choose  m}$ possibilities for $m$, we get 
\begin{align*}
\Prob\left(\sup_{v\in S} \sup_{\tau \in [0, L]} \left(\sum_{t  = t_{e'}}^{t_{e'} + \tau} P_t\right)^T v > c\eta \sqrt{2mL}\right) &\leq {d \choose m}\exp(-c^2)
\end{align*}
Setting $c = a\sqrt{m\log d}$, and upper bounding the expectation by the integral of the tail, we get that
\begin{align*}
\E\left[\sup_{v\in S} \sup_{\tau \in [0, L]} \left(\sum_{t  = t_{e'}}^{t_{e'} + \tau} P_t\right)^T v\right] &\leq O(m\eta \sqrt{L\log d})
\end{align*}
To upper bound the expectation of the sum of this over $E$ epochs, we multiply by the expected number of (sub)epochs, which is $O(\logdel)$ by~\eqref{eq:hp-bpr-expected-E}.
We thus conclude the following upper tail bound on the first sum in~\eqref{eq:proof-hp-bcpr-regret}:
\[
\Prob\left(
\sum_{e \in [E]} \max_{v \in S} \left(\sum_{t \in e} P_t\right)^T v \geq O(m\eta\sqrt{T \logdel \log d}) + u
\right) \leq \exp\left(
-\frac{u^2}{2mT\eta^2}
\right)
\]
and so in particular, conditioned on $A$, with probability at least $1 - \tfrac{\delta}{4}$, this sum is upper bounded by $O(m\eta\sqrt{T\logdel \log d}) + \eta\sqrt{2mT\log\tfrac{4}{\delta}}$.

\par \textbf{Bounding the second sum in~\eqref{eq:proof-hp-bcpr-regret}}. We use the same decomposition trick from~\citep{DevLugNeuCOLT} as we used in the analysis of $\bpr$. Specifically, the sum in question can be written as
\begin{align}
- \sum_{t=1}^{T+1} P_t^T v_{t-1} + \sum_{t=1}^{T+1}  P_t^T\left(v_{t-1} - v_t\right)
\label{eq:bcpr-regret-decomp}
\end{align}
The first sum in~\eqref{eq:bcpr-regret-decomp} is now easily bounded since $v_{t-1}$ and $P_t$ are stochastically independent. In particular, this means each $P_t^Tv_{t-1}$ has distribution $\mathcal{N}(0,\eta^2m)$. Since moreover $\{P_t^Tv_{t-1}\}_{t \in [T+1]}$ are themselves stochastically independent, thus $\sum_{t=1}^{T+1}P_t^Tv_{t-1}$ has distribution $\mathcal{N}(0, \eta^2m(T+1))$ and thus is upper bounded by $\eta\sqrt{2m(T+1)\log\tfrac{8}{\delta}} = O(\eta\sqrt{mT\logdel})$ with probability at least $1 - \tfrac{\delta}{8}$.
\par The second sum in~\eqref{eq:bcpr-regret-decomp} can be upper bounded as follows
\begin{align}
\sum_{t=1}^{T+1} P_t^T\left(v_{t-1} - v_t\right)
&\leq
\sum_{t=1}^{T+1} \|P_t\|_{\infty} \|v_{t-1} - v_t\|_1 \nonumber
\\ &\leq 2m\sum_{t=1}^{T+1} \|P_t\|_{\infty} \mathbf{1}(v_{t-1} \neq v_t) \nonumber
\\ &\leq 2m \left(
\sup_{i \in [d], t \in [T+1]} \left|P_{i,t}\right|
\right) \left( \sum_{t=1}^{T+1} \mathbf{1}(v_{t-1} \neq v_t) \right)\label{eq:bcpr-regret-final}
\end{align}
The first and third inequalities above are due to H\"older's inequality; the second is by triangle inequality and the assumption that each $\|v_t\|_1 \leq m$. Now, $\sup_{i \in [d], t \in [T+1]}|P_{i,t}|$ is the supremum of $d(T+1)$ i.i.d. $\mathcal{N}(0, \eta^2)$ Gaussians and thus by the Borell-TIS inequality is at most $\E[\sup_{i \in [d], t \in [T+1]}|P_{i,t}|] + \sqrt{2\log\tfrac{8}{\delta}} \leq  \eta\sqrt{2 \log (2dT)} +  \eta\sqrt{2\log\tfrac{8}{\delta}} = O(\eta\sqrt{\log\left(\tfrac{dT}{\delta}\right)})$ with probability at least $1 - \tfrac{\delta}{8}$. Thus when this and $A$ occur, the term in~\eqref{eq:bcpr-regret-final} is of order $O(m^2\log d \sqrt{T \logdel} \cdot \eta\sqrt{\log\left(\tfrac{dT}{\delta}\right)})$.
\par The theorem statement now follows by a union bound and combining the above displays.
\end{proof}

\section{Lower bounds on regret}
In this section, we prove the lower bounds in Theorems~\ref{thm:budget-fullinfo-highswitching} and~\ref{thm:budget-fullinfo-lowswitching}.  We first prove that the tails of regret can be no better than sub-Gaussian, implying the lower bound in Theorem~\ref{thm:budget-fullinfo-highswitching}. Then in Section~\ref{app:fullinfo-lowerbound}, we show the lower bound for the low switching budget case.

\subsection{Optimality of sub-Gaussian regret tails}\label{app:subsec:lb-regret}
Recall the classical lower bound of~\cite{CB-expert}, in which the adversary generates all losses $\ell_t(i)$ as independent $\Ber(\half)$ random variables. A simple argument (see e.g. Section~\ref{sec:budgets-pfe}) shows that against this adversary, any algorithm must suffer expected regret of order at least the minimax optimal rate $\Omega(\sqrt{T \log n})$. In fact it is easy to check that this lower bound holds with constant probability, i.e. the independent $\Ber(\half)$ losses sequence forces any algorithm to incur regret $\Omega(\sqrt{T \log n})$ with probabilty at least $1/4$.

We next give a probabilistic analysis of this lower-bound construction and show that for any algorithm, the upper tails of regret are no better than sub-Gaussian. Informally, this order is what one would expect from this lower-bound construction in light of Gaussian isoperimetry and the Borell-TIS inequality. This is formally stated as follows.

\begin{proposition}
\label{prop:regret-tails}
Let $\log \frac 1 \delta$ be $\omega(\log^2 T)$ and $o(T)$. For large enough $n = \Omega(1)$, there exists an oblivious adversary that forces any PFE algorithm to incur at least $\Omega(\sqrt{T\log \frac 1 \delta})$ regret with probability at least $\delta$.
\end{proposition}
\par Note that together the above proposition combined with the expectation lower bound of~\citep{CB-expert}, show the correct dependence in each of the parameters $T$, $n$, and $\delta$ independently but not jointly. 
\par To prove Proposition~\ref{prop:regret-tails}, we will use the following standard result on the concentration and anti-concentration of a binomial random variable. For completeness, we provide a short proof.

\begin{lemma}
\label{lem:bin-tails}
There exist constants $c_1 \leq c_2$ such that for all sufficiently large $T$ and for all $r \in [\tfrac{\log T}{2}, \tfrac{T^{1/2}}{4}]$,
\begin{align*}
\mathbb{P}_{X \sim \Bin(T, \half)}\left(X \geq \frac T 2 + r\sqrt{T}\right) &\leq \exp(-c_1 r^2)\\
\Prob_{X \sim \Bin(T, \half)} \left(X \geq \frac T 2 + r\sqrt{T}\right) &\geq \exp(-c_2 r^2)
\end{align*} 
\end{lemma}
\begin{proof}
The upper bound (concentration) follows immediately from Hoeffding's Inequality. For the lower bound (anti-concentration), 
we will use the fact that $\sqrt{2\pi} n^{n+\half} e^{-n} \leq n! \leq e n^{n+\half} e^{-n}$. Additionally, we use that  $(1+x) \leq \exp(x)$ and that $(1-x) \geq \exp(-2x)$ for $x \in (0, \half)$. To avoid carrying around a ceiling, we assume that $r\sqrt{T}$ is an integer.
\begin{align*}
\Prob\left(X \geq \frac T 2 + r\sqrt{T}\right)
&\geq \Prob\left(X = \frac T 2 + r\sqrt{T}\right)\\
&= \frac{T!}{(\frac T 2 + r\sqrt{T})! (\frac T 2 - r\sqrt{T})!} 2^{-T}\\
&\geq \frac{\sqrt{2\pi} T^{T+\half} e^{-T}}{ 2^T e^2 (\frac T 2 + r\sqrt{T})^{\frac {T+1} 2 + r\sqrt{T}} (\frac T 2 - r\sqrt{T})^{\frac {T+1} 2 - r\sqrt{T}} e^{-T}}\\
&= \frac{2\sqrt{2\pi}}{e^2 \sqrt{T} (1+ \frac {2r}{\sqrt{T}})^{\frac {T+1} 2 + r\sqrt{T}} (1-\frac {2r}{\sqrt{T}})^{\frac{T+1} 2 - r\sqrt{T}}}\\
&= \frac{2\sqrt{2\pi}(1-\frac{2r}{\sqrt{T}})^{r\sqrt{T}}}{e^2 \sqrt{T} (1- \frac{4r^2}{T})^{\frac {T+1} 2}  (1+\frac{2r}{\sqrt{T}})^{ r\sqrt{T}}}\\
&\geq \frac{2\sqrt{2\pi}\exp(-4r^2)}{e^2 \sqrt{T}\exp(-2r^2(1+\frac 1 T))\exp(2r^2)}\\
&\geq \exp(-5r^2).
\end{align*}
\end{proof}
\indent Armed with this lemma, the proof of Proposition~\ref{prop:regret-tails} is straightforward. 
\\ \begin{proof}[Proof of Proposition~\ref{prop:regret-tails}]
Consider the adversary from~\cite{CB-expert}, which generates each loss $\ell_t(i)$ independently as $\Ber(\half)$. Then any algorithm has cumulative loss distributed as $\Bin(T, \half)$. On the other hand, the loss of the best expert is the minimum of $n$ such $\Bin(T, \half)$ random variables.
\par Set $r=\sqrt{\tfrac{\log \frac 1 {2\delta}}{2c_2}}$, and define the events $A := \{\text{loss of algorithm} \geq \frac{T}{2} + 2r\sqrt{T}\}$ and $B := \{\text{loss of best expert} < \frac{T}{2} + r\sqrt{T}\}$. When both $A$ and $B$ occur, the algorithm incurs regret at least $r\sqrt{T} = \Omega(\sqrt{T \logdel})$; thus it suffices to now show $\Prob(A \text{ and } B ) \geq \delta$.
\par The anti-concentration direction of Lemma~\ref{lem:bin-tails} yields
\begin{align*}
\Prob\left(A\right) &\geq 2\delta.
\end{align*} 
On the other hand, the concentration direction of Lemma~\ref{lem:bin-tails} yields
\begin{align*}
\Prob\left(B^C\right) &= \left[\Prob\left(\Bin(T, \half) \geq \frac T 2 + r\sqrt{T}\right)\right]^n
\leq \exp(-c_1 r^2 n)
= (2\delta)^{n\tfrac{c_1}{2c_2}}
\end{align*}
Choosing $n \geq \tfrac{4c_2}{c_1}$, we conclude that for small enough $\delta \leq \tfrac{1}{4}$, then $\Prob(B^C) \leq 4\delta^2 \leq \delta$. The proof is now complete by a union bound:
\[
\Prob\left(A \text{ and } B\right)
= 1 - \Prob\left(A^C \text{ or }  B^C\right)
\geq 1 - \Prob(A^C) - \Prob(B^C)
= \Prob(A) - \Prob(B^C)
\geq 2\delta - \delta
= \delta
\]
\end{proof}

The better of these two lower bounds $\Omega(\sqrt{T \log n})$ and $\Omega(\sqrt{T \log \frac 1 \delta})$ is always $\Omega(\sqrt{T \log \frac n \delta})$ as desired.

\subsection{Lower bounds on regret in low-switching regime}
\label{app:fullinfo-lowerbound}
In this section, we prove the lower bound in Theorem~\ref{thm:budget-fullinfo-lowswitching}.

The idea is essentially a batched version of~\cite{CB-expert}'s classical lower bound for unconstrained PFE. So let us first recall that argument. That construction draws the loss of each expert in each iteration i.i.d. from $\{0,1\}$ uniformly at random. A simple argument shows any algorithm has expected loss $\tfrac{T}{2}$, but that the best expert has loss concentrating around $\tfrac{T}{2} - \Theta(\sqrt{T \log n})$ since (after translation by $\tfrac{T}{2}$) it is the minimum of $n$ i.i.d. simple random walks of length $T$. Therefore they conclude $\E[\Regret] = \Omega(\sqrt{T \log n})$.
\par However, that adversarial construction does not capitalize on the algorithm's limited switching budget in our setting. We accomplish this by increasing the variance of the random walk in a certain way that a switch-limited algorithm cannot benefit from. Specifically, proceed again by batching the $T$ iterations into roughly $E \approx \tfrac{S^2}{\log n}$ epochs, each of uniform length $\tfrac{T}{E}$. For each epoch and each expert, draw a single $\Ber(\half)$ and assign it as that expert's loss for each iteration in that epoch.
\par Informally, the optimal algorithm still incurs expected loss of half for each iteration in epochs it does not switch in; and loss of $0$ for each epoch it switches in. Critically, however, the algorithm can switch at most $S$ times, which is small compared to the number of epochs $E$. Thus any algorithm incurs expected loss roughly $\approx \tfrac{T}{E}\left(\tfrac{E}{2} - S\right) = \tfrac{T}{2} - \Theta\left(\tfrac{T\log n}{S}\right)$. Moreover, the best expert now has loss concentrating around $\frac{T}{E}\left(\tfrac{E}{2} - \Theta(\sqrt{E\log n})\right) = \tfrac{T}{2} - \Theta\left(\tfrac{T\log n}{S} \right)$.
\par Therefore, after appropriately choosing constants in the epoch size, we can then conclude that the expected regret of any $S$-budget algorithm is $\Omega\left(\tfrac{T\log n}{S}\right)$.

We next give details. 
We will make use of the following simple anti-concentration lemma. The proof is standard and ommitted since it follows directly from Lemma 6 of~\citep{CB-expert}, or even just from combining Hoeffding's inequality with a union bound.
\begin{lemma}\label{lem:rand-walk}
There exists a universal constant $c > 0$ such that for all $E, n \in \mathbb{N}_+$
\[
\E\left[
\min_{i \in [n]}Z_i
\right]
\leq \frac{E}{2} -c\sqrt{E\log n}
\]
where $\{Z_i\}_{i \in [n]}$ are i.i.d. $\Bin(E, \half)$. 
\end{lemma}
\begin{proof}[Proof of lower bound in Theorem~\ref{thm:budget-fullinfo-lowswitching}]
Let $c > 0$ be the constant from Lemma~\ref{lem:rand-walk}. We will restrict WLOG to the case $\tfrac{c}{2} \log n \leq S \leq \tfrac{c}{2}\sqrt{T \log n}$. Indeed, when the latter inequality does not hold, then the lower bound from Theorem~\ref{thm:budget-fullinfo-highswitching} applies. And when the former inequality does not hold, then we may apply the $\Omega(T)$ lower bound that we show presently for (the easier setting of) $S' = \tfrac{c}{2}\log n$.
\par Mini-batch the $T$ iterations into $E := \frac{4}{c^2}\frac{S^2}{\log n}$ epochs, each of uniform length $\frac{T}{E}$. For each epoch $e \in [E]$, assign to each expert $i \in [n]$ a loss of $X_e(i) \sim \Ber(\half)$ for each iteration in that epoch. Clearly this adversary is oblivious.
\par Note that the cumulative loss $\sum_{t=1}^T \ell_t(i)$ of each expert $i$ is equal in distribution to $\tfrac{T}{E}$ times a $\Bin(E, \half)$ r.v. Thus by Lemma~\ref{lem:rand-walk},
\[
\E\left[
\min_{i \in [n]} \sum_{t=1}^T \ell_t(i)
\right]
\leq
\frac{T}{E}
\left(
\frac{E}{2} - c \sqrt{E \log n}
\right)
= \frac{T}{2} - c\frac{T\sqrt{\log n}}{\sqrt{E}}
= \frac{T}{2} - \frac{c^2}{2}\frac{T\log n}{S}
\]
Now let us compute the expected loss any algorithm $\calA$ that uses at most $S$ switches. It is simple to see that the following deterministic strategy is optimal: for each epoch, burn the first iteration by not moving; then if we are on a good expert do not move for the rest of the epoch; else if we are on a bad expert then make a switch if we have switches remaining. To analyze this, let the random variable $B$ denote the number of epochs in which the algorithm plays a bad expert in that epoch's first iteration. Then the r.v. $\min(B, S)$ is equal to the number of bad epochs in which $\calA$ makes a switch. Thus since $\E[B] = \tfrac{E}{2}$, we obtain
\begin{align*}
\E\left[
\text{cumulative loss of }\calA
\right]
&= \E\left[
1\cdot \min(B,S) + \frac{T}{E}\cdot \left( B - \min(B,S)\right)
\right]
\\ &\geq \frac{T}{E}\E\left[B - S
\right]
\\ &\geq \frac{T}{E} \left(
\frac{E}{2} - S
\right)
\\ &= \frac{T}{2} - \frac{TS}{E}
\\ &= \frac{T}{2} - \frac{c^2}{4}\frac{T\log n}{S}
\end{align*}
Combining the two above displays, we conclude that any $S$-switching budget algorithm $\calA$ suffers expected regret at least $ \frac{c^2}{4}\frac{T\log n}{S} = \Omega\left(\tfrac{T \log n}{S} \right)$.
\end{proof}

This implies the bound on the expected regret. The high probability regret bound follows in a similar fashion by mini-batching the lower bound argument in Section~\ref{app:subsec:lb-regret}.


\section{Proof 2 of lower bound in Theorem~\ref{thm:mab}: via direct modification of~\citep{DekDinKorPer}'s multi-scale random walk}\label{app:mab-lowerbound-construction}
This proof is significantly more involved than the first proof given in Subsection~\ref{subsec:mab-lb-reduction}, but it yields an explicit adversarial construction. Since the proof relies on (existing) sophisticated techniques, we first outline the main ideas and tools.

\subsection{Motivation via adaptation of the construction in~\citep{CBDekSha13}.}\label{app:mab-lowerbound-construction-CB} Let us begin by showing how to adapt the pioneering lower-bound construction for switching-cost MAB in~\citep{CBDekSha13}, to our switching-budget setting. Although this idea does not quite work (due to reasons stated below about the losses drifting outside of $[0,1]$), it will motivate the adaptation of~\citep{DekDinKorPer} we later describe in Appendix Subsections~\ref{app:mab-lowerbound-construction-Dek} and~\ref{app:mab-lowerbound-proof-CB} (that does work).
\par The critical idea in their construction of losses is the use of random walks to hide the best action. Formally, draw an action $i^* \in [n]$ uniformly at random; it will be designated as the best action. Then define the oblivious losses
\[
\ell_t(i) := \sum_{\tau=1}^t Z_\tau - \eps \mathbf{1}(i = i^*)
\]
where $Z_1, \dots, Z_T \sim \mathcal{N}(0, 1)$ are i.i.d. standard Gaussians. Note that action $i^*$ is better than all other actions by a deterministic amount $\eps$ in each iteration, but the identity of $i^*$ is hidden (at least partially) because the algorithm receives only bandit feedback.
\par The key intuition of this loss construction is that the player learns absolutely nothing from playing the same arm in consecutive iterations. Therefore an optimal algorithm will switch between arms in each of the first $S$ iterations (the ``exploration'' phase), and subsequently play the arm estimated to be best for the remaining $T-S$ iterations (the ``exploitation'' phase). A standard information theoretic argument shows that it takes roughly $\Omega\left(\eps^{-2}\right)$ switches to distinguish whether a given arm is $\eps$-biased~\citep{Auer02,CBDekSha13}. Moreover, a simple averaging argument shows that at least one of the $n$ actions is played at most $\frac{S}{n}$ times in these first $S$ exploratory iterations. Informally, this shows that the player cannot identify the best arm after $S$ exploratory iterations when $\eps = o(\sqrt{\tfrac{n}{S}})$, and therefore the minimax rate is lower bounded by $\sup_{\eps = o\left(\sqrt{\tfrac{n}{S}}\right)}\Omega(T \eps) = \Omega\left(\tfrac{T\sqrt{n}}{\sqrt{S}}\right)$.
\par However, the problem with this adversarial construction is that the losses are certainly not bounded within $[0,1]$ and indeed are likely to drift to very large values that scale with $T$. As such, it is not clear whether the above lower bound is merely an artifact of these large losses. 

\subsection{Adversarial construction via adaptation of the construction in~\citep{DekDinKorPer}.}\label{app:mab-lowerbound-construction-Dek} 
We fix this issue of bounded losses by closely following the elegant argument of~\citep{DekDinKorPer}, who gave the first rigorous tight lower bound for switching-cost MAB. Their construction is similar in flavor to~\citep{CBDekSha13}'s random walk construction described above; however, they prevent drifting by generating the losses instead from a carefully chosen \textit{multi-scale random walk} (MRW).
\par Our loss functions will be identical to the one in~\citep{DekDinKorPer} (see Figure 1 in their paper), except that we will alter the bias $\eps$ of the best arm based on the switching-budget $S$. Specifically, we will set $\eps = \frac{\sqrt{n}}{54\sqrt{S} (\log T)^{3/2}}$, as opposed to their choice $\eps = \frac{n^{1/3}}{9 T^{1/3} \log T}$. For completeness, we re-state this construction as follows in our notation and with our $\eps$.
\begin{figure}[htbp]
\caption{Explicit construction of (random, oblivious) adversarial loss sequence that forces any $S$-switching-budget algorithm to incur $\E[\Regret] \geq \min\left(T, \tilde{\Omega}\left(\frac{T\sqrt{n}}{\sqrt{S}}\right)\right)$.} 
\label{fig:mab-lowerbound-construction}
\begin{framed}
\begin{itemize}[topsep=5pt,itemsep=2pt,parsep=0pt,partopsep=0pt]
       \item Set $\eps := \frac{\sqrt{n}}{54\sqrt{S} (\log_2 T)^{3/2}}$ and $\sigma := \frac{1}{9 \log_2 T}$
       \item Choose $i^* \in [n]$ uniformly at random
       \item Draw $Z_1, \dots, Z_T \sim \mathcal{N}(0, \sigma^2)$ i.i.d. Gaussians
       \item Define $W_0, \dots, W_T$ recursively by:
       \begin{align*}
       W_0 &:= 0
       \\ W_t &:= W_{p(t)} + Z_t \; \; \; \forall t \in [T]
       \end{align*}
       where $p(t) := t - 2^{\delta(t)}$ and $\delta(t) := \max \{ i \geq 0 \; : \; 2^i \text{ divides } t \}$
       \item For all $t \in [T]$ and $i \in [n]$, define 
       \begin{align*}
       \ell_t^{\text{unclipped}}(i) &:= W_t + \half - \eps \cdot \mathbf{1}(i = i^*)
       \\ \ell_t(i) &:= \text{clip}(\ell_t^{\text{unclipped}}(i)) 
       \end{align*}
       where clip$(x) := \min(\max(x, 0), 1)$
    \end{itemize}
\end{framed}
\end{figure}
\\  We now provide a bit of intuition about this construction, and refer the reader to~\citep{DekDinKorPer} for further details and intuition.
\par In essence, this construction has many similarities to the one in Subsection~\ref{app:mab-lowerbound-construction-CB}: the best action is better than all others by a constant gap $\eps = \tilde{\Theta}\left(\sqrt{\frac{n}{S}}\right)$, and this best action is hidden by constructing the losses from a certain random walk. A similar  heuristic information-theoretic argument as in Subsection~\ref{app:mab-lowerbound-construction-CB} above, shows that any algorithm needs more than $S$ switches to distinguish the best arm. We make this argument formal in the following subsection.
\par The key difference in this construction is in the so-called \textit{parent function} $p(t)$. Note that defining instead $p(t) = t-1$ would recover the construction in Subsection~\ref{app:mab-lowerbound-construction-CB} (modulo the clipping of losses and setting of parameters $\epsilon$ and $\delta$). However as pointed out above, then $W_t = \sum_{\tau=1}^t Z_\tau$ would often drift outside of $[0,1]$. It turns out that the choice of $p(t) := t - 2^{\delta(t)}$ ensures that the resulting stochastic process $W_t$ will have small ``depth'' and ``width''. We refer to~\citep{DekDinKorPer} for formal definitions and further details about this, and just remark here that informally these properties ensure that (1) w.h.p. the process $W_t$ does not drift far, implying that w.h.p. the losses $\ell_t(i)$ are not clipped; and (2) each switch gives the player little information about the identity of the best arm.
\par In words, the definition of $p(t)$ means that $W_t$ is created by summing up the $Z_i$ in the binary expansion of $t$: e.g., $W_1 = Z_1$, $W_2 = Z_2$, $W_3 = Z_2 + Z_3$, $W_4 = Z_4$, $W_5 = Z_4 + Z_1$, $W_6 = Z_6 + Z_4$, and so on; see Figure 2 of~\citep{DekDinKorPer} for details. The functions $\ell_t^{\text{unclipped}} : [n] \to \Real$ are then created from centering the process $W_t$ at $\half$, and then the losses $\ell_t : [n] \to [0,1]$ are created from clipping the outputs to $[0, 1]$. Intuitively, the larger $\sigma$ is, the more it masks the bias $\epsilon$ of the best arm; but the smaller it is, the more likely the losses will not need to be clipped.
\par We refer the reader to~\citep{DekDinKorPer} for further intuition and details, and now proceed to formally prove the desired lower bound in Theorem~\ref{thm:mab} using this construction.

\subsection{Proof of lower bound}\label{app:mab-lowerbound-proof-CB}
We now prove the lower bound in Theorem~\ref{thm:mab} using the adversarial construction in Figure~\ref{fig:mab-lowerbound-construction}. The proof will almost exactly follow the analysis in~\citep{DekDinKorPer}. As they do, let us first prove the result for \textit{deterministic} algorithms; extending to \textit{randomized} algorithms will then be easy at the end. Formally, we will aim to first show the following. 

\begin{lemma}\label{lem:mab-lowerbound-deterministic}
The loss sequence in Figure~\ref{fig:mab-lowerbound-construction} forces any deterministic $S$-budget algorithm $\calA$ to incur expected regret at least
\[\E[\Regret_T(\calA)] \geq \min\left(T, \frac{1}{324 (\log T)^{3/2}}\left(\frac{T\sqrt{n}}{\sqrt{S}}\right)\right)\]
\end{lemma}

Let $\{i_t\}_{t \in [T]}$ be the decisions of $\calA$. Since $\calA$ is deterministic, we know each $i_t$ is a deterministic function of its previous observations $\{\ell_\tau(i_\tau) \}_{\tau \in [t-1]}$.
\paragraph*{Step 1 of proof: Compare regret to unclipped regret.} It will be easier mathematically to analyze the \textit{unclipped} regret, which is defined exactly like regret but on the unclipped losses $\elltunclipped$. This quantity is more amenable to analysis since it equal to the following simple expression
\[
\Regretunclipped(\calA)
:= \sum_{t=1}^T \elltunclipped(i_t) - \min_{i^* \in [n]} \sum_{t=1}^T \ell_t(i^*)
= \epsilon\left(T-N_{i^*}\right)
\]
where for each $i \in [n]$, $N_i$ denotes the number of times $\calA$ played action $i$.
\par The first step in the proof is thus to compare $\E[\Regret(\calA)]$ to $\E[\Regretunclipped(\calA)]$. This is achieved by the following lemma, whose statement and proof are nearly identical to that of Lemma 4 in~\citep{DekDinKorPer}; we provide details for completeness.
\begin{lemma}[Slight modification of Lemma 4 in~\citep{DekDinKorPer}]\label{lem:mab-lowerbound-regret}
If $T > 6$ and $S \geq \frac{n}{81 \log^3 T}$, 
\[
\E\left[\Regret(\calA)\right]
\geq 
\E\left[\Regretunclipped(\calA)\right] - \frac{\eps T}{6}
\]
\end{lemma}
\begin{proof}
Define the event $B := \{\forall t \in [T]: \; \ell_t = \elltunclipped \}$; we will first show $\Prob(B) \geq \tfrac 5 6$. To do this we show that the stochastic process $W_t$ has small drift. Indeed, Lemmas 1 and 2 of~\citep{DekDinKorPer} show that (with setting the parameter $\delta := \tfrac 1 T \leq \tfrac 1 6$)
\[
\Prob\left(\max_{t \in [T]} |W_t| \leq \frac{1}{3}\right) \geq \frac{5}{6}
\]
Whenever this occurs, we have that $\half + W_t \in [\tfrac 1 6, \tfrac 5 6]$ for all $t \in [T]$; and thus since $\eps \leq \tfrac 1 6$ (by our assumption on $S$), we have that all unclipped losses $\elltunclipped(i) \in [0,1]$. Thus $\Prob(B) \geq \tfrac 5 6$.
\par To conclude, observe that $\Regret(\calA) = \Regretunclipped(\calA)$ when $B$ occurs. Otherwise, we have always have
\[
0 \leq \Regret(\calA) \leq \Regretunclipped(\calA) \leq \eps T
\]
The first inequality is because there is an action which is always the best; the second inequality is because the gap to the best action can only decrease when losses are clipped; and the final equality is since the best action for the unclipped losses is always best by a constant gap of $\eps$. Therefore we conclude $ \Regretunclipped(\calA) - \Regret(\calA) \leq \eps T$ and so the proof is concluded by a simple conditioning argument:
\[
\E\left[\Regretunclipped(\calA) - \Regret(\calA)\right]
=
\Prob(B^C)\E\left[\Regretunclipped(\calA) - \Regret(\calA) \; | \; B^C\right]
\leq \frac{\eps T}{6}
\]
\end{proof}

\paragraph*{Step 2 of proof: Analyze unclipped regret in terms of the algorithm's ability to distinguish the best arm.} 
At this point, we need to define some new notation. Following~\citep{DekDinKorPer}, denote by $\calF$ the $\sigma$-algebra generated by the player's observations up to time $T$. Also for $i \in [n]$, denote by $Q_i(\cdot) := \Prob(\cdot |i^* = i)$ the conditional probability measures on the event that the best action is $i$. Similarly denote by $Q_0(\cdot)$ the probability measure in which no action is good (``$i^* = 0$''). 
Let $\E_{Q_i}[\cdot]$ denote expectations w.r.t. these probability measures. 
Finally, we will denote by $\|P - Q\|_{\text{TV},\calF} := \sup_{A \in \calF}|P(A) - Q(A)|$ the total variation distance between two probability measures $P$ and $Q$ with respect to the sigma-algebra $\calF$.
\par The key lemma of this section is then to lower bound the expected unclipped regret in terms of how well the algorithm can distinguish the best arm. As is standard, the latter quantity will be measured in terms of the algorithm's ability at the end of the game, to distinguish whether there was a biased arm or not, i.e. whether the losses were generated from the measure $Q_i(\cdot)$ or $Q_0(\cdot)$. This is made formal as follows.
\begin{lemma}[Slight modification of Lemma 5 in~\citep{DekDinKorPer}]\label{lem:mab-lowerbound-unclipped}
\[
\E\left[\Regretunclipped(\calA)\right]
\geq
\frac{\eps T}{2} - \frac{\eps T}{n}\sum_{i=1}^n \left\|Q_0 - Q_i\right\|_{\text{TV}, \calF}
\]
\end{lemma}
\begin{proof}
Identical to the proof of Lemma 5 in~\citep{DekDinKorPer}, except without switching costs.
\end{proof}

\paragraph*{Step 3 of proof: Upper bound the algorithm's ability to distinguish the best arm, in terms of its number of switches.} Recall the discussions in Subsections~\ref{app:mab-lowerbound-construction-CB} and~\ref{app:mab-lowerbound-construction-Dek} about how the random-walk loss construction ensures that the amount of information the algorithm learns about the best action, is controlled by the number of switches it makes. The following makes this intuition precise. 
\begin{lemma}[Corollary 1 of~\citep{DekDinKorPer}]\label{lem:mab-lowerbound-TV}
\[
\frac{1}{n}\sum_{i=1}^n \left\|Q_0 - Q_i\right\|_{\text{TV}, \calF} \leq \frac{\eps}{\sigma \sqrt{n}}\sqrt{\E_{Q_0}[\Switches_T(\calA)] \cdot \log_2 T}
\]
\end{lemma}
We note that the above is the main technical lemma in the proof. We refer to the original paper of~\citep{DekDinKorPer} for its proof, and just remark here that roughly speaking the argument follows the standard framework of: upper bounding total variation by KL divergence via Pinsker's inequality; and then upper bounding KL divergence via properties of the loss construction. 
\par Critical to us will be the trivial observation that $\E_{Q_0}[\Switches_T(\calA)] \leq S$ since $\calA$ deterministically \textit{never} makes more than $S$ switches.

\paragraph*{Step 4 of proof: Combining everything together to prove the lower bound against deterministic algorithms.} We are now ready to prove Lemma~\ref{lem:mab-lowerbound-deterministic}. 
\\ \begin{proof}[Proof of Lemma~\ref{lem:mab-lowerbound-deterministic}]
We may assume WLOG that $T > 6$ and $S \geq \frac{n}{81 (\log_2 T)^3}$. The former condition can be justified by simply enlarging the constant in our final $\Omega(\cdot)$ lower bound. The latter can be justified since when it does not hold, the resulting setting is only harder than when $S = \frac{n}{81 (\log_2 T)^3}$; and for this setting, the ensuing argument shows a lower bound of $\tfrac{T}{36} = \Omega(T)$.
\par Therefore by applying Lemmas~\ref{lem:mab-lowerbound-regret},~\ref{lem:mab-lowerbound-unclipped}, and~\ref{lem:mab-lowerbound-TV}, and using the fact that $\calA$ is limited to $S$ switches,
\[
\E\left[\Regret(\calA)\right]
\geq 
\frac{\eps T}{3}
- \frac{\eps^2 T}{\sigma \sqrt{n}} \sqrt{S \cdot \log_2 T}
\]
The proof is complete by plugging in our choice of parameters $\epsilon$ and $\sigma$.
\end{proof}

\paragraph*{Step 5 of proof: Extending the hardness result to randomized algorithms.} We are now finally ready to prove the lower bound in Theorem~\ref{thm:mab}. We use a standard argument that is similar to the one in the proof of Theorem 1 of~\citep{DekDinKorPer}. 
\\ \begin{proof}[Proof of lower bound in Theorem~\ref{thm:mab}]
Lemma~\ref{lem:mab-lowerbound-deterministic} shows that the oblivious adversary defined by the losses in Figure~\ref{fig:mab-lowerbound-construction} is hard against any \textit{deterministic} algorithm. Now since the adversary is oblivious, and since any randomized algorithm can be viewed as a distribution over deterministic algorithms (where all coin flips done before the game begins), the expected regret of a randomized algorithm against the losses in Figure~\ref{fig:mab-lowerbound-construction} can be computed by first taking the expectation over the algorithm's internal randomness.
\end{proof}
We note there is also a standard argument if one wants a hard \textit{deterministic} adversary. Since the adversary is oblivious, applying the Max-Min inequality and the probabilistic method yields
\begin{align*}
\max_{\text{random adversary}} \min_{\text{random } S\text{-budget algorithm}} \E[\Regret]
&\leq
\min_{\text{random } S\text{-budget algorithm}} \max_{\text{random adversary}} \E[\Regret]
\\ &= \min_{\text{random } S\text{-budget algorithm}} \max_{\text{deterministic adversary}} \E[\Regret]
\end{align*}
Now Theorem~\ref{thm:mab} lower bounds the first inequality of the above display. Thus we conclude that for any $S$-budget algorithm, there exists a hard \textit{deterministic} adversary. Of course this is now an existential result not an explicit construction since the deterministic adversary now depends on the algorithm it is trying to be hard against.

\section{On the upper tails of standard algorithms}
\label{app:lb-kv-sd}
In this section, we discuss whether existing algorithms achieve h.p. guarantees. As far as we know, there are three existing algorithms that in expectation achieve the minimax optimal rate of $O(\sqrt{T\log n})$ for both regret and number of switches. A natural question is whether these algorithms also achieve these optimal rates w.h.p. 
\par The first (chronologically) of these three algorithms is ~\citep{KalVem}'s Multiplicative Follow the Perturbed Leader algorithm ($\mfpl$). It seems to be folklore that this algorithm's upper tail is far too large (in fact it is inverse polynomially large!) for both switching and regret to achieve h.p. bounds. However, we are not aware of anywhere in the literature that this is explicitly written down, so for completeness we give proofs of these facts in Appendix~\ref{app:subsec:kv}.
\par The second of these algorithms is the Shrinking Dartboard ($\sd$) algorithm proposed by~\citep{Dartboard}. Their paper does not consider whether $\sd$ achieves h.p. bounds on either switching or regret; and to the best of our knowledge, neither of the questions is answered in the literature yet. In Appendix~\ref{app:subsec:sd}, we give a simple proof that $\sd$ does achieve h.p. bounds on switching (even against adaptive adversaries!). However, we give a simple construction for which $\sd$ achieves no better than sub-exponential regret tails. It is not clear though whether $\sd$ achieves sub-exponential regret tails in general -- this seems a hard problem and would have interesting implications since the tails would be uniform (as opposed to our proposed algorithms; see discussion immediately following Theorem~\ref{thm:hp-sc}). Nevertheless, in light of the sub-Gaussian upper tails achieved by Theorem~\ref{thm:hp-sc} and the sub-Gaussian lower bound in Proposition~\ref{prop:regret-tails}, $\sd$ is anyways provably suboptimal. Moreover, $\sd$ does not have an efficient implementation for online combinatorial optimization~\citep{DevLugNeu15}, whereas, as shown in Section~\ref{app:hp-oco}, our Framework~\ref{alg:framework} easily extends to this setting.
\par The third of these three algorithms is the Prediction by Random-Walk Perturbation algorithm from~\citep{DevLugNeu15}. However, analyzing its upper tails (for both switching and regret) seems quite difficult and was left as an open problem in their paper.
\par Finally, we discuss briefly in Section~\ref{app:fll} why~\citep{KalVem}'s Follow the Lazy Leader (FLL) algorithm, which gives expected regret and switching bounds for the combinatorial setting, does not satisfy the conditions of our framework.
\subsection{$\mfpl$ achieves h.p. bounds neither for switches nor for regret}\label{app:subsec:kv}
Let us first recall the algorithm: before the game starts, the algorithm draws a perturbation\footnote{Note that for clarity, we use exponential perturbations instead of Laplacian perturbations as in~\citep{KalVem}'s original paper, since this is how we discussed the algorithm in Section~\ref{sec:hp}. An identical argument works for Laplacian perturbations, since Laplacians also have sub-exponential tails.} $P_i \sim \tfrac{\exp(1)}{\eps}$ for each action; then at each iteration $t \in [T]$, the algorithm plays the action 
\[i_t := \argmin_{i \in [n]} \left(P_i + \sum_{s < t} \ell_s(i)\right)\]
that is best with respect to the perturbed cumulative losses.~\citep{KalVem} show that when $\eps$ is chosen of order $\sqrt{\tfrac{\log n}{T}}$, then in expectation $\mfpleps$ achieves the minimax optimal rate of $O(\sqrt{T \log n})$ for both switching and regret. However, this rate is achieved w.h.p. neither for switches nor for regret, since both of their upper tails are only inverse polynomially small (instead of inverse exponentially small). This is formally stated as follows.

\begin{proposition}\label{prop:lb-kv}
Consider PFE with $n=2$ actions. Let $\eps = \Theta\left(\tfrac{1}{\sqrt{T}}\right)$, so that in expectation $\mfpleps$ achieves the minimax regret rate of $O(\sqrt{T})$. There exists a deterministic oblivious adversary such that for any $T$ sufficiently large,
\[
\Prob\left(
\Regret(\mfpleps), \;\Switches(\mfpleps) = \Omega(T)
\right)
\geq \Omega\left(
\frac{1}{\sqrt{T}}
\right)
\]
\end{proposition}

To prove Proposition~\ref{prop:lb-kv}, we will employ the folklore construction traditionally used for showing $\Omega(T)$ switching and regret lower bounds on the na\"ive Follow the Leader (FTL) algorithm. For FTL, this construction is deterministically hard (i.e. with probability $1$); here, we show that for $\mfpleps$ the construction is hard with non-negligible probability.
\\ \begin{proof}
By assumption $\eps = \tfrac{c}{\sqrt{T}}$. Take any $T \geq \tfrac{c^2}{4}$ and define the losses as follows. Set $\ell_1(i) := \half \cdot 1(i = 2)$, and in every subsequent iteration $t > 1$ set $\ell_t(i) := \begin{cases}
1(i = 1) & t \text{ even} \\
1(i = 2) & t \text{ odd}
\end{cases}$. The key observation is that whenever $P_1 - P_2 \in [0, \half)$, then $\Switches(\mfpl) = T$ and moreover $\Regret = (\text{loss of player}) - (\text{loss of best expert}) = (T - \half) - (\tfrac{T}{2} - \half) = \tfrac{T}{2}$. Thus it suffices to lower bound the probability that $P_1 - P_2 \in [0, \half)$. But this is a straightforward calculation: recalling that each $P_i := \tfrac{R_i}{\eps} = R_i\tfrac{\sqrt{T}}{c}$ where $R_i$ are i.i.d. standard exponential random variables,
\begin{align*}
\Prob\left(
P_1 - P_2 \in [0, \half)
\right)
&= \Prob\left(
R_1 - R_2 \in [0, \tfrac{c}{2\sqrt{T}})\right)
\\ &= \int_{0}^{\infty} e^{-r_2} \int_{r_2}^{r_2+ \tfrac{c}{2\sqrt{T}}} e^{-r_1} dr_1 dr_2
\\ &= \left(1 -e^{-\tfrac{c}{2\sqrt{T}}}\right)\int_{0}^{\infty} e^{-2r_2} dr_2
\\ &= \half \left(1 - e^{-\tfrac{c}{2\sqrt{T}}}\right)
\\ &\geq \tfrac{c}{8\sqrt{T}}
\end{align*}
where the final step is due to the inequality $1 - e^{-x} \geq \tfrac{x}{2}$ which holds for all $x \in [0, 1]$. 
\end{proof}

\subsection{$\sd$ achieves h.p. bounds for switches, but cannot achieve sub-Gaussian regret tails}\label{app:subsec:sd}
\citep{Dartboard} show that their $\sd$ algorithms achieve in expectation the minimax optimal rate $O(\sqrt{T \log n})$ for both switches and regret, when $\sd$'s learning parameter $\eta$ is chosen of order $\tfrac{\sqrt{\log n}}{T}$. However, they do not consider whether this optimal rate is achieved with high probability.
\par First, we give a simple proof that $\sd$ w.h.p. achieves this optimal rate for switching.
\begin{proposition}
Let $\eta = c \sqrt{\frac{\log n}{T}}$. Then for any (even adaptive) adversary and any $\delta \in (0, 1)$, 
\[
\Prob\left(
\Switches_T(SD) \geq c\sqrt{T \log n} + \sqrt{2T\logdel}
\right)
\leq \delta
\]
\end{proposition}
\begin{proof}
Let $X_t := 1(i_t \neq i_{t-1})$ denote the indicator r.v. that $\sd$ switches actions between iterations $t-1$ and $t$. Define also $Z_t$ to be the indicator r.v. that line $7$ in Algorithm 2 of~\citep{Dartboard} is executed at iteration t. Then conditional on any event $A$ in the sigma-algebra generated by past decisions of the player and adversary $\sigma_{t-1} := \sigma(i_{1}, \ell_1, \dots, i_{t-1}, \ell_{t-1})$,
\[
\Prob\left(
X_t = 1\; | \; A
\right)
\leq
\Prob\left(
Z_t = 1 \; | \; A
\right)
= 1 - (1 - \eta)^{\ell_{t-1}(i)}
\leq \eta
\]
The first inequality is because $Z_t$ stochastically dominates $X_t$ conditional on any historic event $A$, since line $7$ must be executed in order for $\sd$ to switch actions. The middle equality is by definition of $\sd$. The final inequality is due to the $\ell_{\infty}$ constraint on the losses to lie within $[0,1]$.
\par We conclude that $\{M_t := \sum_{s=1}^t X_s - t\eta \}_{t \in [T]}$ is a super-martingale w.r.t. the filtration $\{\sigma_{t-1}\}_{t \in [T]}$. Moreover, it has bounded differences of at most $1$, since $|M_{t+1} - M_t| \overset{\text{a.s.}}{\leq} \max(\eta, 1 - \eta) \leq 1$. Since clearly $M_0 = 0$, we conclude from Hoeffding-Azuma's inequality~\citep{Concentration} that
\[
\Prob\left(
\Switches(\sd) \geq \eta T + r
\right)
=
\Prob\left(
M_T - M_0 \geq r
\right)
\leq \exp\left(
-\frac{r^2}{2T}
\right)
\]
The proof is complete by setting $r = \sqrt{2T\logdel}$.
\end{proof}

Next, we show a sub-exponential lower bound on the upper tails of $\sd$'s regret.
\begin{lemma}\label{lem:sdregret}
The regret of $\sd$ does not have sub-Gaussian tails, even when there are only $n=2$ actions. That is, for all $\delta \in (0, 1)$, there exists an oblivious adversary that forces $\sd$ to incur at least $\min\left(\Omega(\sqrt{T} \logdel), T\right)$ regret with probability at least $\delta$.
\end{lemma}
\begin{proof}
Let $\eta = \tfrac{c}{\sqrt{T}}$ so that in expectation $\sd$ achieves the minimax regret rate of $O(\sqrt{T})$. 
Consider the following simple adversarial construction: draw a ``bad'' arm $i^*$ uniformly at random, and construct the losses $\ell_t(i) := 1(i = i^*, \; t \leq T')$ where 
$T' := \tfrac{\log \tfrac{1}{2\delta}}{2\eta} + 1$.\footnote{This is well-defined when $T' \leq T$. If $\delta$ is small enough that $T' > T$, the written proof works without modification after replacing $\delta$ with a larger $\delta'$ such that $T' = T$. This then proves a stronger statement than required for $\delta$.}
\par Clearly if $i_1 = \dots = i_{T'} = i^*$ then the algorithm incurs regret of $T'$ which is of order $\Omega(\sqrt{T}\logdel)$ since $\eta$ must be chosen of order $\Theta(T^{-1/2})$ so that $\sd$ achieves the minimax rate of $O(\sqrt{T})$ regret in expectation. Thus it suffices to show $\Prob(i_1 = \dots = i_{T'} = i^*) \geq \delta$. To see this, note that $i_1 = i^*$ with probability $\tfrac{1}{n} = \half$; and moreover by the definition of $\sd$, $i_{t+1} = i_t$ with probability at least $1 - \eta$ regardless of $\sd$'s previous actions $i_1,\dots,i_{t-1}$. Thus by the inequality $1 - \eta \geq e^{-2\eta}$ for $\eta \in [0, \tfrac{3}{4}]$, we conclude that for all sufficiently large $T = \Omega(1)$,
\begin{align*}
\Prob\left(i_1 = \dots = i_{T'} = i^*\right)
\geq \half (1 - \eta)^{T' - 1}
\geq \half \exp\left(-2\eta(T' - 1)\right)
= \delta
\end{align*}
\end{proof}
As a specific instantiation, this shows that $\sd$ incurs $\Omega(T)$ regret with probability at least $2^{-O(\sqrt{T})}$.

\subsection{On FLL for online combinatorial optimization}
\label{app:fll}
The Follow the Lazy Leader (FLL) algorithm of~\citep{KalVem} is an elegant variant of $\mfpl$ that easily extends to the online combinatorial optimization setting and admits bounds on expected switches as well as expected regret. Unfortunately, the proof of these bounds does not quite go via bounding the regret in terms of number of switches. Thus this algorithm does not satisfy property (ii) that our framework requires, and it is not clear if one can convert an FLL-type algorithm to get a high probability result.

\end{document}